\documentclass{article}

\usepackage{microtype}
\usepackage{graphicx}
\usepackage{subcaption}
\usepackage{booktabs} 

\usepackage{hyperref}

\usepackage[accepted]{icml2026}

\usepackage{amsmath}
\usepackage{amssymb}
\usepackage{mathtools}
\usepackage{amsthm}

\usepackage{tikz}
\usetikzlibrary{bayesnet,positioning}
\usepackage{algorithm}
\usepackage{algorithmic}

\usepackage[capitalize,noabbrev]{cleveref}

\theoremstyle{plain}
\newtheorem{theorem}{Theorem}[section]
\newtheorem{proposition}[theorem]{Proposition}
\newtheorem{lemma}[theorem]{Lemma}
\newtheorem{example}[theorem]{Example}
\newtheorem{corollary}[theorem]{Corollary}
\theoremstyle{definition}
\newtheorem{definition}[theorem]{Definition}
\newtheorem{assumption}[theorem]{Assumption}
\theoremstyle{remark}

\crefname{assumption}{Assumption}{Assumptions}
\Crefname{assumption}{Assumption}{Assumptions}

\newcommand{\EE}{{\mathbb{E}}}

\usepackage[textsize=tiny]{todonotes}

\icmltitlerunning{OPE for Missingness-Aware Policies in MDPs with Rewards Missing Not at Random}

\begin{document}

\twocolumn[
  \icmltitle{Off-Policy Evaluation for Missingness-Aware Policies in \\
  MDPs with Rewards Missing Not at Random}



  \icmlsetsymbol{equal}{*}

  \begin{icmlauthorlist}
    \icmlauthor{Ziheng Wei}{umich}
    \icmlauthor{Annie Qu}{ucsb}
    \icmlauthor{Rui Miao}{utd}
  \end{icmlauthorlist}

  \icmlaffiliation{umich}{Department of Statistics, University of Michigan at Ann Arbor}
  \icmlaffiliation{ucsb}{Department of Statistics and Applied Probability, University of California at Santa Barbara}
  \icmlaffiliation{utd}{Department of Mathematical Sciences, University of Texas at Dallas}

  \icmlcorrespondingauthor{Rui Miao}{rui.miao@utdallas.edu}

  \icmlkeywords{Causal Inference, Fitted-Q-Evaluation, Markov Decision Processes, Missing Data, Off-Policy Evaluation, Reinforcement Learning, Shadow Variable}

  \vskip 0.3in
]



\printAffiliationsAndNotice{}  

\begin{abstract}
In offline Reinforcement Learning, immediate rewards in logged batch data are often unobserved due to sparse or irregular record-keeping, or censored beyond certain reward values. This issue arises in practical settings, including health care and marketing. We investigate off-policy evaluation (OPE) in finite-horizon Markov decision processes when rewards are missing not at random (MNAR), which breaks ignorability and induces selection bias even after conditioning on states and actions. To address this,  we formalize a reward-dependent propensity model and use future states as shadow variables to identify the full-data conditional mean reward. We further introduce  a bridge function that recovers the conditional mean reward without explicitly modeling the MNAR mechanism, and estimate it via a min-max procedure to avoid double sampling. Building upon these identification results, we propose an Fitted-Q-Evaluation-style estimator that  propagates the recovered rewards while allowing target policies to depend on past missingness indicators. Finally, we establish consistency and finite-sample error bounds for our OPE estimator, and show through experiments the strong performance of our method compared to existing methods on simulated and MIMIC-III Sepsis data.
\end{abstract}

\section{Introduction}

Reinforcement Learning (RL) has achieved remarkable successes in sequential
decision-making domains ranging from robotics to healthcare, and most recently
in large language models and AI agents \citep{ouyang2022training,rafailov2023direct,achiam2023gpt}. However, learning optimal policies
often requires vast amounts of interaction with the environment, which can be
costly, risky, and even unethical in high-stakes real-world applications such as
healthcare and education \citep{tsiatis2019dynamic,murphy2003optimal,mandel2014offline}. In these applications, we often have interaction
data collected according to some behavior policy, e.g., standard care or usual
strategy. Estimating the value of a target policy using historical datasets
collected under a different behavior policy, a problem known as Off-Policy Evaluation (OPE), is essential in offline RL \citep{levine2020offline,uehara2022review,voloshin2019empirical,wang2024fine}. Accurate OPE is critical for deploying
safe and effective policies without the need for dangerous online exploration.

However, in many practical scenarios, such as medical treatment or digital
advertising, the data is plagued by unobserved factors or missing values \citep{little2019statistical,kallus2020confounding}.
A particularly pervasive challenge arises when rewards are missing not at random (MNAR), under which the
probability of observing a reward depends on the latent value of the reward itself. For instance, in health care with scheduled chronic-disease management, patient outcomes such as patient-reported quality-of-life scores may be missing due to incomplete questionnaires or skipped follow-ups, while subsequent clinical states (lab results, vital signs, encounter records) are still recorded in the electronic health record (EHR). The recording of the outcome may itself depend on its value: patients with worsening symptoms are more likely to skip self-reported assessments, while those feeling well may also under-report. Similarly, in
multi-touch digital marketing, attribution is frequently disrupted by privacy
limitations. While non-conversions (zero rewards) are trivially observable,
high-value purchases often involve cross-device journeys, such as a user clicking
on mobile but converting on a desktop, or trigger manual review flows that break
attribution links. Consequently, high-value conversions go missing while
low-value or null outcomes remain fully observed, creating a dataset that
systematically biases the learning process against the most desirable outcomes. Similar MNAR feedback has been systematically studied in recommender systems, where popularity and exposure biases make logged interactions MNAR and consequently distort offline evaluations \citep{yang2018unbiased}.

Standard OPE methods, such as Fitted Q-evaluation (FQE) and Importance Sampling (IS), rely on fully observed trajectories. Failing to account for the
missingness could lead to biased OPE and hence sub-optimal decision-making. In the OPE literature, scheme of missingness has been studied in various
aspects. Partially Observable Markov Decision Processes  \citep[POMDPs,][]{kaelbling1998planning,jaakkola1995reinforcement} consider the case
where the state is partially observed, which can be viewed as a special case of
missingness. However, most existing POMDP methods assume that certain state variables are totally unobserved. In general, off-policy value in POMDPs are often unidentifiable without strong assumptions \citep{tennenholtz2020off,bennett2021off,shi2022minimax,uehara2023future}.
Some works define the missingness of certain status (e.g., hitting wall in Gridworld 
environment \citep{sutton1998reinforcement}) in the reward function by assigning it to be some large negative 
values to discourage certain actions, which lead to certain states with voided rewards \citep{ng1999policy,devlin2012dynamic}. 
However, this approach may not accurately reflect the true reward structure and can introduce bias.
\citet{chu2023multiply, park2025evaluating} study the OPE problem with truncated
trajectories, where they treat missingness as certain constraints. However, by 
penalizing on the missingness, these methods may shift the policy evaluation away 
from the true potential reward without missingness. \citet{wang2025off} propose an inverse
probability weighting method for OPE with nonignorable truncation, but their method relies 
on an extra shadow variable, or requires expert knowledge to select such a variable from observed states.

In this paper, we study the OPE problem in MDPs with MNAR rewards to estimate values of 
target policies accounting for the past missingness. MNAR rewards break standard ignorability assumptions as 
the reward-dependent missingness induces selection bias and confounds state-action returns.
The challenge is to recover the value of a target policy when the observed trajectories systematically 
underreport high or low rewards and when the missingness itself can depend on the past action and state, all without online data to re-collect or intervene.

We address these issues by formalizing the reward MNAR mechanism via a reward-dependent
propensity score model and leveraging future states as shadow variables.
Under mild completeness conditions, the shadow variables allow us to 
identify the full-data conditional mean reward even when the reward is MNAR. In addition, we introduce 
a bridge function \(b_t(S_t,A_t,S_{t+1})\) satisfying \(\EE\left\{b_t(S_t,A_t,S_{t+1})\mid R_t,S_t,A_t\right\} = R_t\), 
enabling recovery of the conditional mean reward without explicitly estimating the MNAR mechanism.
This avoids the variance blow-up in inverse propensity weighting. We propose the min-max optimization to estimate the bridge function
and the value function, which avoids the double sampling issue.

Building on these identification results,
we further develop an FQE-style estimator that integrates the bridge function and allows 
target policies to depend on the previous missingness indicators. The procedure 
propagates the recovered rewards through the Bellman recursion of the target policy, 
yielding stable value estimates. We further establish the consistency and finite-sample 
error bounds of the proposed estimator in nonparametric settings. Extensive experiments 
on simulated data and a MIMIC-III dataset demonstrate the effectiveness of our method compared to existing benchmarks.


\section{Preliminaries}
We consider an episodic Markov Decision Process (MDP) $\mathcal{M} = \{\mathcal{S}, \mathcal{A}, \mathcal{P}, r, T\}$, where $\mathcal{S}$ and $\mathcal{A}$ denote the state and action spaces, respectively. The horizon length $T$ is finite, and we assume the terminal state $S_{T+1}$ is observed. The transition kernels $\mathcal{P} = \{P_t\}_{t=1}^T$ govern the state dynamics, where $P_t: \mathcal{S} \times \mathcal{A} \to \Delta(\mathcal{S})$ maps state-action pairs to distributions over next states. The reward functions $r=\{r_t\}_{t=1}^T$ are defined as conditional expectations given the next state: $r_t(s, a, s') = \EE[R_t \mid S_t = s, A_t = a, S_{t+1} = s']$ for any $(s, a, s') \in \mathcal{S} \times \mathcal{A} \times \mathcal{S}$. We assume bounded rewards $R_t \in \mathcal{R} \subseteq [-1, 1]$.

We introduce an observation indicator $O_t \in \{0, 1\}$, where $O_t = 1$ indicates that the reward $R_t$ is observed at time $t$, and $O_t = 0$ otherwise. Importantly, we allow the rewards to be missing not at random (MNAR), that is, even after conditioning on current states and actions, the missingness probability may depend on the possibly unobserved reward itself. We formalize this through the propensity score $e_t(s, a, r) = P(O_t = 1 \mid S_t = s, A_t = a, R_t = r)$ for $t = 1, \ldots, T$.
\begin{assumption}[No Future Dependence]\label{ass:nofuture}
For all $t = 1, \ldots, T-1$,
\[
O_t \perp (S_{t+1:T}, R_{t+1:T}) \mid S_t, A_t, R_t.
\]
\end{assumption}
This assumption states that the current missingness indicator $O_t$, given the current state, action, and reward, is independent of all future states and rewards. For example, in healthcare, whether a patient's health outcome is recorded typically depends on their current condition, not on future events that have not yet occurred.

Our goal is to evaluate the performance of a target policy $\pi = \{\pi_t\}_{t=1}^T$. We allow $\pi_t$ to depend on the previous reward missingness, which is practically relevant when decisions adapt based on whether prior outcomes were observed. For example, in healthcare, a clinician may choose a more conservative treatment when the previous lab result is missing, whereas the historical standard-care behavior policy acts only on the current clinical state. The two policies thus share the same underlying causal structure but differ in which variables they condition on at decision time. Formally, $\pi_t: \mathcal{S} \times \{0, 1\} \to \Delta(\mathcal{A})$ with $\pi_t(a \mid s, o_-) = P(A_t = a \mid S_t = s, O_{t-1} = o_-)$. To accommodate this dependence, we define an augmented state $\widetilde{S}_t = (S_t, O_{t-1}) \in \widetilde{\mathcal{S}} = \mathcal{S} \times \{0, 1\}$, so that the target policy can be written as $\pi_t(a \mid \widetilde{S}_t)$. The augmented process must satisfy the Markov property below.

\begin{assumption}[Markov Property for Augmented Process]\label{ass:markov}
The augmented process $\{(\widetilde{S}_t, A_t)\}_{t=1}^T$ with $\widetilde{S}_t = (S_t, O_{t-1})$ is an MDP 
\[
P(\widetilde{S}_{t+1} \mid \widetilde{S}_{1:t}, A_{1:t}) = P(\widetilde{S}_{t+1} \mid \widetilde{S}_t, A_t), \quad t = 1, \ldots, T.
\]
\end{assumption}
This assumption ensures that augmenting the state with the previous missingness indicator preserves the Markov property, and allows the value function recursion to hold with augmented state. The augmented transition kernel is $P(\widetilde{S}_{t+1} \mid \widetilde{S}_t, A_t) = P((S_{t+1}, O_t) \mid S_t, A_t)$. We set $O_0 = 0$ and let $S_1 \sim \rho_1$ denote the initial state distribution. The Q-function and value function satisfy the Bellman equation
\begin{equation}\label{eq:bellman}
\begin{aligned}
Q_t^\pi(s,a)=&\mathbb E\big[r_t(s,a,S_{t+1})+V_{t+1}^\pi(S_{t+1},O_t)\mid S_t=s,\\&A_t=a\big],\\
V_t^\pi(s,o_-)
=&\sum_a \pi_t(a\mid s,o_-)Q_t^\pi(s,a),\quad V_{T+1}^\pi\equiv 0.
\end{aligned} 
\end{equation}

The policy value is defined as $V(\pi)\equiv\mathbb E_{\widetilde S_1\sim\tilde\rho_1}\big[V_1^\pi(\widetilde S_1)\big] = E_{S_1\sim\rho_1}\big[V_1^\pi(S_1,0)\big]$.

In OPE, data are collected under a behavior policy $\pi^b = \{\pi_t^b\}_{t=1}^T$, where $\pi_t^b: \mathcal{S} \to \Delta(\mathcal{A})$ does not depend on the missingness indicators. The observed dataset $\mathcal{D}$ consists of $n$ i.i.d.\ trajectories $\tau_i = \{S_{t,i}, A_{t,i}, O_{t,i}, R_{t,i}^{\mathrm{obs}}, S_{t+1,i}\}_{t=1}^T$ for $i = 1, \ldots, n$, where $R_{t,i}^{\mathrm{obs}} = O_{t,i} \cdot R_{t,i}$ denotes the observed reward (zero when missing). See \cref{fig:dag1} for a directed acyclic graph (DAG) illustrating the data-generating process.

\begin{figure}[ht]
  \begin{center}
\begin{tikzpicture}[scale=0.75, transform shape, >=stealth, node distance=1cm, semithick, latent/.append style={circle, draw, minimum size=8mm, inner sep=0pt}]
    \node[latent] (S)  at (0,0)    {$S_t$};
    \node[latent] (A)  at (2,0)    {$A_t$};
    \node[latent] (Sp) at (4,0)    {$S_{t+1}$};
    \node[latent] (Ap)  at (6,0)    {$A_{t+1}$};
    \node[latent] (Spp) at (8,0)    {$S_{t+2}$};
    \node[latent] (R)  at (2,2)  {$R_t$};
    \node[latent] (O)  at (4,2) {$O_t$};
    \node[latent] (Rp)  at (6,2)  {$R_{t+1}$};
    \node[latent] (Op)  at (8,2) {$O_{t+1}$};

    \draw[->] (S)  -- (A);
    \draw[->] (A)  -- (Sp);
    \draw[->] (Sp)  -- (Ap);
    \draw[->] (Ap)  -- (Spp);

    \draw[->] (S)  -- (R);
    \draw[->] (A)  -- (R);
    \draw[->] (Sp)  -- (R);
    \draw[->] (Sp)  -- (Rp);
    \draw[->] (Ap)  -- (Rp);
    \draw[->] (Spp)  -- (Rp);

    \draw[->, bend right=25] (S) to (Sp);
    \draw[->, bend right=25] (Sp) to (Spp);
    \draw[->] (A)  -- (O);
    \draw[->] (R)  -- (O);
    \draw[->] (S)  -- (O);
    \draw[->] (Sp)  -- (Op);
    \draw[->] (Ap)  -- (Op);
    \draw[->] (Rp)  -- (Op);
    \draw[->, color=blue, very thick] (O)  -- (Ap);
  \end{tikzpicture}
  \caption{DAG for the data-generating process. Black arrows represent the standard MDP dynamics and the MNAR reward mechanism, shared by both policies. The \textcolor{blue}{blue} arrow $O_t \to A_{t+1}$ is specific to the target policy, which is allowed to depend on the previous missingness indicator at decision time; the behavior policy depends only on the current state. }
  \label{fig:dag1}
  \end{center}
\end{figure}
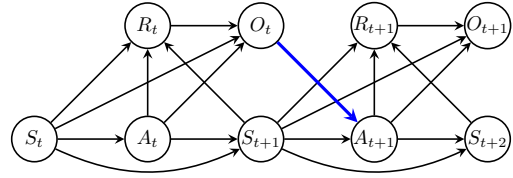
To deal with distribution shift in OPE, concentrability coefficients are often introduced \cite{munos2003error,munos2007performance,chen2019information,le2019batch,duan2021risk}. We define concentrability coefficient $\kappa_t$ at time $t$ in the following assumption.
\begin{assumption}[Concentrability]\label{ass:concentrability}
    Given target policy $\pi$ and behavior policy $\pi^b$, for each $t=1,\dots,T$, assume there exist finite constants $\{\kappa_t\}_{t=1}^T$ such that
\[\|\frac{\pi_t(a\mid s,o_-)}{\pi_t^b(a\mid s)}\|_\infty
\le \kappa_t.\]
Equivalently, for all $(s,o_-, a)$ with $\pi_t^b(a\mid s)>0$,
$\pi_t(a\mid s,o_-)\le \kappa_t\,\pi_t^b(a\mid s)$. 
\end{assumption}
\begin{assumption}\label{ass:boundconc}
    Assume there exist constants $a>0$ and $\alpha_t\ge \alpha>1$ such that for all
$t\in\{1,\dots,T\}$, \[\kappa_t \le 1 + \frac{a}{t^{\alpha_t}}.\]
\end{assumption}
\cref{ass:boundconc} is an enhanced version for \cref{ass:concentrability}, which controls the cumulative growth of the concentrability coefficients so that the action mismatch does not compound over time. A structurally similar decay condition appears in Assumption 2 of \citet{miao2022off} in a POMDP setting.
\begin{corollary}[Bounded cumulative concentrability]
    Under \cref{ass:concentrability,ass:boundconc}, for $t=1,\dots, T$, 
    \[(\prod_{j=1}^t \kappa_j)^{1/2} \le \exp\Big(\frac12\sum_{j=1}^t \frac{a}{j^{\alpha_j}}\Big) \le\exp\!\big(\frac a2\zeta(\alpha)\big):= K,\]
    where $\zeta(\cdot)$ is the Riemann zeta function. 
\end{corollary}

\section{Identification}\label{sec:identification}
In this section, we establish identification results for the policy value, and formalize the conditions required for our approach.

To address the challenges posed by MNAR data in causal inference, \citet{miao2015identification,miao2016varieties} propose identification methods with the help of auxiliary variables called \emph{shadow variables}. Inspired by this line of research, we establish a nonparametric value-based approach for policy value identification. Our key insight is to adopt the next state $S_{t+1}$ as the shadow variable, which serves as a proxy that helps recover information about the unobserved rewards.
The shadow variable must satisfy two conditions that govern its relationship with the reward and missingness indicator. 
\begin{assumption}[Exclusion Restriction]\label{ass:exclusion}
  Suppose for all $t=1,\dots, T$,  $S_{t+1}$ satisfies
  \[S_{t+1}\perp O_t\mid R_t,S_t,A_t.\]
\end{assumption}
\begin{assumption}[Relevance Condition]\label{ass:relevance}
  Suppose for all $t=1,\dots, T$,  $S_{t+1}$ satisfies
  \begin{equation*}
      S_{t+1}\not\perp R_t \mid S_t, A_t, O_t = 1.
  \end{equation*}
\end{assumption}
The two assumptions above are basic conditions for $S_{t+1}$ to be a valid shadow variable at time $t$. \cref{ass:exclusion} is a direct consequence of \cref{ass:nofuture}, which shows that conditional on the current state-action pair and the (possibly unobserved) reward, $S_{t+1}$ provides no additional information about whether the reward is observed. \cref{ass:relevance} ensures that on the observed subset, $S_{t+1}$ remains informative about $R_t$ beyond what is already captured by $(S_t, A_t)$. This condition guarantees that the shadow variable carries useful information about the reward. 
In the causal graph in \cref{fig:dag1}, \cref{ass:exclusion} is consistent with $d$-separation: conditioning on $(R_t,S_t,A_t)$ blocks all paths from $S_{t+1}$ to $O_t$. For subsequent analysis, we define the extended propensity score of non-missingness $e_t(s,a,r,s')= P(O_t=1\mid S_t=s,A_t=a,R_t=r,S_{t+1} = s')$. The choice of $S_{t+1}$ aligns with the recurring idea in POMDPs, i.e., leveraging future states or observations to serve as proxy latent state information \cite{singh2003learning,uehara2023future,xu2023instrumental,zhang2024curses}. Moreover, $S_{t+1}$ is endogenous to the MDP and already recorded in the logged transitions, so it requires no auxiliary measurement beyond the standard offline data, in contrast to approaches that specify an extra shadow variable \citep{wang2025off}.

For identification, rather than explicitly modeling the missingness mechanism under MNAR, our goal is to recover the full-data conditional mean reward $\mathbb E[R_t\mid S_t,A_t]$ using only observable quantities. Under MNAR, the observed reward conditional expectation $\mathbb E\big[R_t\mid S_t=s, A_t=a, O_t=1\big]$ generally differs from the target $\mathbb E\big[R_t\mid S_t=s, A_t=a\big]$, and directly using observed rewards would lead to a biased policy evaluation.

We adopt a bridge-based imputation strategy for missing rewards motivated by proximal causal inference \cite{tchetgen2020introduction,cui2024semiparametric}. Related bridge constructions also appear in the literature on confounded POMDPs \cite{miao2022off,shi2022minimax,hong2023policy,li2025reinforcement}. The core idea is to construct functions that can learn the missing rewards in an unbiased manner by exploiting the relationship between rewards and next states.

Specifically, we introduce a sequence of bridge functions $\{b_t: \mathcal{S\times A\times S}\rightarrow\mathbb{R}\}_{t=1}^T$ satisfying the moment condition that 
\begin{equation}\label{eq:bridge}
    \mathbb E[b_t(S_t,A_t,S_{t+1})\mid R_t,S_t,A_t] = R_t,\quad a.s.
\end{equation}
\cref{eq:bridge} links the target quantity of interest to the observable offline distribution, and converts the recovery of $\mathbb E[R_t\mid S_t,A_t]$ into an estimable conditional moment problem.

Taking conditional expectation of \eqref{eq:bridge} given $(S_t,A_t)$ yields
\begin{equation}\label{eq:bridgecond}
    \begin{aligned}
        &\mathbb E[b_t(s,a,S_{t+  1})\mid S_t=s,A_t=a] \\
        =& \mathbb E(R_t\mid S_t=s,A_t=a):=\bar r_t(s,a).
    \end{aligned}
\end{equation}
Thus, the bridge function reproduces the correct one-step conditional mean reward required by the Bellman recursion. 

Moreover, a crucial observation which enables practical estimation is that
\[P(S_{t+1}\mid R_t,S_t,A_t,O_t=1) = P(S_{t+1}\mid R_t,S_t,A_t),\]
by \cref{ass:nofuture}.
This implies that the bridge moment condition in \cref{eq:bridge} can be identified from the observed subset $\{O_t=1\}$. This means we can estimate the bridge function $b_t$ using only samples where rewards are observed, and then evaluate it at samples with $O_t=0$ to impute the missing rewards.

We further introduce the following assumptions for the identification of the policy value.
\begin{assumption}[Positivity]\label{ass:positivity}
For all $t=1,\dots, T$, and for all $(s,a,r)\in \mathcal {S\times A\times R}$,
    \(0<e_t(s,a,r)<1\).
\end{assumption}
The positivity assumption ensures that every state-action-reward triple has a positive probability of being both observed and unobserved, which is commonly used in causal inference literature.
\begin{assumption}[Completeness]\label{ass:complete}
For all $(s,a)\in \mathcal S\times \mathcal A$, $t=1,\dots, T$, 
\begin{enumerate}
    \item[(1)] For any square-integrable function $h$,
        \begin{align*}
            &\mathbb E[h(R_t)\mid S_t = s, A_t = a, S_{t+1}] \\= &\int h(R_t)p(R_t\mid s,a,S_{t+1})dR_t = 0, \quad a.s.
        \end{align*}
if and only if $h(R_t)=0,\quad a.s.$;
    \item[(2)] For any square-integrable function $g$,
    \begin{align*}
      &\mathbb E[g(S_{t+1})\mid R_t, S_t = s, A_t = a]\\ = &\int g(S_{t+1})p(S_{t+1}\mid R_t, s, a)dS_{t+1} = 0,\quad a.s.  
    \end{align*}
    if and only if $g(S_{t+1}) = 0,\quad a.s.$
\end{enumerate}
\end{assumption}
\cref{ass:complete} guarantees the existence and uniqueness of the bridge functions $b_t$, for $t=1,\dots, T$. Completeness assumptions are standard in the proximal causal inference \cite{tchetgen2020introduction,cui2024semiparametric}, where they ensure that the conditional expectations are sufficiently rich to identify the target functional. 

To provide concrete intuition, we characterize completeness in the tabular setting.
\begin{example}[Completeness in tabular setting]
    Assume $\mathcal S,\;\mathcal R$ are tabular. Let matrix $M_{t,s,a}\in\mathbb R^{|\mathcal S|\times|\mathcal R|}$ where $
M_{t,s,a}(s',r)
:=P\left(S_{t+1}=s'\mid R_t=r, S_t=s, A_t=a\right), \quad s'\in\mathcal S,r\in\mathcal R$. Then
    \begin{enumerate}
        \item If for all $(t,s,a)$,  $rank(M_{t,s,a}) = |\mathcal R|$, then \cref{ass:complete}~(1) holds, and hence the bridge exists;
        \item If for all $(t,s,a)$, $rank(M_{t,s,a}) = |\mathcal S|$, then \cref{ass:complete}~(2) holds, and hence the bridge is unique;
        \item If $|\mathcal S|=|\mathcal R|$, then $M_{t,s,a}$ is invertible for all $(t,s,a)$ if and only if both \cref{ass:complete}~(1) and \cref{ass:complete}~(2) hold.
    \end{enumerate}
\end{example}
Then we give the identification results for policy value as follows:
\begin{theorem}[Policy value identification]
  \label{thm:identification}
  For an augmented MDP satisfying \cref{ass:nofuture,ass:markov,ass:relevance,ass:positivity,ass:complete} and some regularity conditions, there always exist bridges $\{b_t\}_{t=1}^T$ that satisfy \cref{eq:bridge}, and the policy value then can be identified using $\{b_t\}_{t=1}^T$. 
\end{theorem}
See \cref{app:proof-ident} for other regularity conditions and proof. 

Based on \cref{thm:identification}, we develop a value-based approach for policy value identification, which circumvents modeling the missing mechanism explicitly, in contrast to approaches such as \citet{miao2016varieties, miao2018identification, wang2025off}. This is practically significant as it can avoid the high variance induced from IPW methods or requires strong parametric assumptions about the missingness. Our identification procedure consists of three steps.

\textbf{Step 1 (Learn $b_t$)}.
For each $t=1,\dots,T$, learn $b_t$ from
\[
\mathbb E\big[b_t(s,a,S_{t+1})\mid R_t=r,S_t=s,A_t=a\big]=r,
\]
using the observed subset with $O_t=1$. This step leverages the shadow variable structure to extract reward information from state transitions.

\textbf{Step 2 (Identify $Q^\pi$ and $V^\pi$).}
Define the imputed reward
\begin{equation}\label{eq:impreward}
    \widetilde R_t := O_tR_t+(1-O_t)b_t(S_t,A_t,S_{t+1}),
\end{equation}
and solve the Bellman recursion
\begin{align*}
Q_t^\pi(s,a) =& \mathbb E[\widetilde R_t+V_{t+1}^\pi(S_{t+1},O_t)\mid S_t=s,A_t=a],\\
V_t^\pi(s,o_-) =& \sum_a \pi_t(a\mid s,o_-)Q_t^\pi(s,a),\quad V_{T+1}^\pi\equiv 0.
\end{align*}
It is trivial to verify that 
\begin{equation}\label{eq:unbiased}
    \mathbb E[\widetilde R_t\mid R_t,  S_t, A_t] =R_t,\quad a.s.
\end{equation}
\textbf{Step 3 (Identify the policy value).}
Compute $V(\pi)=\mathbb E_{\widetilde S_1\sim\widetilde\rho_1}[V_1^\pi(\widetilde S_1)]$ by backward induction.

\section{Estimation}
In this section, we discuss estimation of policy value and propose a FQE-style estimation method. 
To estimate the policy value, it suffices to estimate the bridge functions $\{b_t\}$ from conditional moment models 
\begin{equation}\label{eq:popbridge}
    \mathbb E[b_t(S_t,A_t,S_{t+1})-R_t\mid R_t,S_t,A_t,O_t=1]=0,\quad a.s.,
\end{equation}

which can be viewed as nonparametric instrumental variable (NPIV) problems. 
A natural approach would be to directly minimize the squared conditional moment
\[
\min_{b_t\in\mathcal B^{(t)}}
\mathbb{E}\Bigl[\bigl(\mathbb{E}[b_t(S_t, A_t, S_{t+1})-R_t \mid R_t, S_t, A_t, O_t =1]\bigr)^2\Bigr].
\]
However, this is not implementable for a single batch of trajectories because the squared conditional moments can lead to the double-sampling issue \cite{baird1995residual, sutton1998reinforcement}. To circumvent the double-sampling problem, we adopt a min-max estimator for $b_t$ \cite{dikkala2020minimax}. The key insight is to replace the squared conditional moment with a saddle-point formulation that can be estimated from a single batch of samples.

For each time step $t$, we solve
\begin{equation}\label{eq:minimax}
    \begin{aligned}
        \min_{b_t\in\mathcal B^{(t)}}&\sup_{g\in\mathcal G^{(t)}}\frac1{n_t}\sum_{i\in \mathcal I_t^{\mathrm{obs}}}\big[(b_t(S_{t,i},A_{t,i},S_{t+1,i})
    -R_{t,i})g_{t}(R_{t,i},\\&S_{t,i},A_{t,i})\big]-\lambda(\|g_t\|_{\mathcal G^{(t)}}^2+\frac{U}{\delta^2}\|g_t\|_2^2)+\lambda\mu\|b_t\|_{\mathcal B^{(t)}}^2,
    \end{aligned}
\end{equation}
where $\mathcal I_t^{\mathrm{obs}}=\{i\in\{1,\dots,n\}: O_{t,i}=1\}$ denotes the observed dataset at time $t$, and $n_t = |\mathcal I_t^{\mathrm{obs}}|$. We denote the function classes of $g_t$ and $b_t$ by $\mathcal G^{(t)},\;\mathcal B^{(t)}$, which can be chosen as finite dimensional linear spaces, and infinite dimensional spaces like RKHSs, neural networks, etc. We focus on RKHSs in this paper. Let $\mathcal Q^{(t)}$ be the RKHS containing function $Q_t$. The term $\lambda \frac{U}{\delta^2}\|g_t\|_2^2$ is the $L_2$ penalty on the critic function $g_t$. The norms $\|\cdot\|_{\mathcal G^{(t)}}^2, \; \|\cdot\|_{\mathcal B^{(t)}}^2, \; \|\cdot\|_{\mathcal Q^{(t)}}^2$ denote the functional norm associated with $\mathcal G^{(t)},\;\mathcal B^{(t)},\;\mathcal Q^{(t)}$. $\lambda, U, \delta, \mu >0$ are tuning parameters for the penalties. 

Then, we can substitute the estimates into fitted-Q-evaluation (FQE) algorithm and obtain the estimate of policy value $\widehat {V}(\pi)$. See Algorithm \ref{alg:prox-fqe-noncons} for the point-estimated policy value estimation algorithm, where we use penalized nonparametric least squares to learn $Q_t$:
\begin{equation}\label{eq:fitq}
    \widehat Q_t = \arg\min_{f\in\mathcal Q^{(t)}}
\frac1{n}\sum_{i=1}^n
\big(f(S_{t,i},A_{t,i})-y_{t,i}\big)^2
+\lambda_{Q,t}\|f\|_{\mathcal Q^{(t)}}^2, 
\end{equation}
where $y_{t,i}$ is defined in \cref{alg:prox-fqe-noncons}.

\begin{algorithm}[t]
\caption{Proximal FQE algorithm}
\label{alg:prox-fqe-noncons}
\begin{algorithmic}
\STATE {\bfseries Input:} Offline dataset
$\mathcal D=\{\tau_i\}_{i=1}^n$, where $\tau_i=\{(S_{t,i},A_{t,i},O_{t,i},R_{t,i}^{\mathrm{obs}},S_{t+1,i})\}_{t=1}^T$,
target policy $\pi=\{\pi_t\}_{t=1}^T$, horizon $T$,
function classes $\{\mathcal B^{(t)},\mathcal G^{(t)},\mathcal Q^{(t)}\}$.
\STATE {\bfseries Initialize:} $\widehat V_{T+1}^\pi(\cdot,\cdot)\gets 0$.
\FOR{$t=T$ {\bfseries down to} $1$}
  \STATE {\bfseries Bridge fitting:}
  Obtain $\hat b_t$ by solving \cref{eq:minimax} on $\mathcal I_t^{\mathrm{obs}}$.
  \STATE {\bfseries Imputation:} for all $i=1,\dots,n$ set $\widehat{\widetilde R}_{t,i}\gets R_{t,i}^{\mathrm{obs}}+(1-O_{t,i})\hat b_t(S_{t,i}, A_{t,i}, S_{t+1,i})$.
  \STATE {\bfseries Targets for Bellman regression:}
  \IF{$t<T$}
    \STATE $y_{t,i}\gets \widehat{\widetilde R}_{t,i}+\widehat V^\pi_{t+1}(S_{t+1,i},O_{t,i})$, $i=1,\dots,n$.
  \ELSE
    \STATE $y_{t,i}\gets \widehat{\widetilde R}_{t,i}$, $i=1,\dots,n$.
  \ENDIF 
  \STATE {\bfseries Fit $Q_t$:} regress $y_{t,i}$ on $(S_{t,i},A_{t,i})$ by \cref{eq:fitq} to obtain $\widehat Q_t$. 
  \STATE {\bfseries Define $V_t^\pi$:} $\widehat V^\pi_{t}(s,o_-)\gets \sum_{a}\pi_t(a\mid s,o_-)\widehat Q_t(s,a)$.
\ENDFOR
\STATE {\bfseries Output:} $\widehat V(\pi)\gets \frac1n\sum_{i=1}^n\widehat V_{1}^\pi(S_{1,i},0)$.
\end{algorithmic}
\end{algorithm}

\section{Theoretical results}
In this section, we establish consistency and finite-sample estimation error bounds for bridges $\hat b_t$ and the policy value.
\subsection{Preliminaries}
\begin{definition}[Local Rademacher Complexity \cite{bartlett2005local}]
  
  For any function class $\mathcal G$ defined over random variable $X$ and radius $\delta>0$, the \textit{local Rademacher complexity} is 
  \[\mathcal R_{n}(\mathcal G,\delta)
=\mathbb E_{\varepsilon, X}\Big[\sup_{g\in\mathcal G:\|g\|_{2}\le \delta}\big|\frac1{n}\sum_{i=1}^{n}\varepsilon_i g(X_i)\big|\Big],\]
where $\{X_i\}$ are i.i.d. samples of $X$ and $\{\varepsilon_i\}$ are Rademacher random variables. $\|g\|_2^2:= \mathbb E[g(X)^2]$ is the $L_2$ norm of function $g$. 
\end{definition}
Suppose the function class $\mathcal G$ satisfies
\begin{enumerate}
\item \textit{symmetric}, if $g\in\mathcal G$ then $-g\in\mathcal G$;
\item \textit{star-shaped}, if $g\in\mathcal G$ then $rg\in\mathcal G$ for all $r\in[0,1]$;
\item \textit{$b$-uniformly bounded}, $\|g\|_{\infty}:=\sup_{x\in\mathcal X}|g(x)|\le b$ for all $g\in\mathcal G$.
\end{enumerate}
Then, the critical radius of such function class $\mathcal G$, denoted by $\delta_n$, is the smallest solution to the inequality $\mathcal R_{n}(\mathcal G,\delta)\le \frac{\delta^2}{b}$.

\subsection{Bridge function estimation error bound}
For notational simplicity, define the projection operator $\mathcal T_t: \mathcal L^2\{\mathcal{S\times A\times S}\}\to \mathcal L^2\{\mathcal {R\times S\times A}\}$, which satisfies \[\mathcal T_t b_t = \mathbb E[b_t(S_t, A_t,S_{t+1})|R_t, S_t, A_t].\]

\begin{assumption}[Boundedness of $\mathcal T_t$]\label{ass:operatorbound}
    For any $b_t\in \mathcal B^{(t)}$, $\mathcal T_tb_t\in \mathcal G^{(t)}$, and there exists $L>0$ such that \[\|\mathcal T_t b_t\|_{\mathcal G^{(t)}}\le L\|b_t\|_{\mathcal B^{(t)}}.\]
\end{assumption}
\begin{assumption}[Realizability]\label{ass:realize}
    Suppose the true bridge function $b_t^*$ lies in function class $\mathcal B^{(t)}$. Similarly, we also assume $Q_t^\pi \in \mathcal Q^{(t)}$.
\end{assumption}  
In practice, the boundedness assumption often holds when the conditional distribution of $S_{t+1}$ given $(R_t, S_t, A_t)$ is sufficiently smooth. Also, Realizability is a standard assumption in nonparametric estimation, requiring that the function classes are rich enough to contain the true targets. 

Next, we define the $B_t$-bounded norm subset of $\mathcal B^{(t)}$ as $\mathcal B^{(t)}_{B}:=\{b_t\in\mathcal B^{(t)}:\ \|b_t\|_{\mathcal B^{(t)}}\le B_t\}$ and $U_t$-bounded norm subset of $\mathcal G^{(t)}$ as $\mathcal G^{(t)}_{U}:=\{g_t\in\mathcal G^{(t)}:\ \|g_t\|_{\mathcal G^{(t)}}\le U_t\}$. 

\begin{assumption}[Richness of test function class]\label{ass:richness}
    We suppose the test function approximation error within subset $\mathcal G^{(t)}_{L^2\|(b-b_t^*)\|_{\mathcal B^{(t)}}^2}$ is bounded by 
\[\sup_{b\in\mathcal B_B^{(t)}}\inf_{g_t\in \mathcal G^{(t)}_{L^2\|b-b_t^*\|_{\mathcal B^{(t)}}^2}}
\|g_t-\mathcal T_t(b-b_t^*)\|_2
\le \eta_t<\infty.\]
This shows that function class $\mathcal G^{(t)}$ is rich enough so that $\mathcal T_t(b-b_t^*)$ admits an $L_2$-approximation within $\mathcal G^{(t)}_{L^2\|(b-b_t^*)\|_{\mathcal B^{(t)}}^2}$ uniformly over $b\in\mathcal B_B^{(t)}$. 
\end{assumption}

Now we are ready to analyze min-max estimator $\hat b_t$ estimated by \cref{eq:minimax}. 
\begin{theorem}[Projected bridge estimation error bound \cite{dikkala2020minimax,miao2022off}]\label{thm:projbridgeerr}
For any $t=1,\dots, T$, suppose function class $\mathcal G^{(t)}$ is star-shaped and symmetric. Suppose $\mathcal G^{(t)}$ and $\mathcal B^{(t)}$ are $1$-uniformly bounded. Define product class
\begin{equation*}
    \begin{aligned}
        \mathcal J^{(t)}_{B,U}:=&
\Big\{((s,a,s'),(r,s,a))\mapsto \alpha(b_t(s,a,s')-b_t^*(s,a,\\&s'))g^{U}_{b,t}(r,s,a)\mid b_t-b_t^* \in \mathcal B^{(t)}_{B}, \alpha\in[0,1]\Big\},
    \end{aligned}
\end{equation*}
where $g^{U}_{b,t} = \arg\min_{g_t\in \mathcal G_U^{(t)}}\|g_t-\mathcal T_t(b-b_t^*)\|_2$. Define the two critical radii for function class $\mathcal G^{(t)}_{3U}$ and $\mathcal J^{(t)}_{B,L^2B}$, namely $\delta_t^{\mathcal G}$ and $\delta_t^{\mathcal J}$, and their maximum $\delta_{n_t}:=\max\{\delta_t^{\mathcal G},\delta_t^{\mathcal J}\}$. Let $\delta_t = \delta_{n_t}+c_0\sqrt{\frac{\log(c_1/\zeta)}{n_t}}$.
Assume $\eta_t\lesssim \delta_t$, then if $\lambda \asymp \frac{\delta_t^2}{U}$ and $\mu\ge \frac43L^2+\frac{36}{B_t\lambda}\delta_t^2$, we have with probability $1-3\zeta$, the following bound holds
\[\big\|\mathcal T_t\big(\hat b_t-b_t^*\big)\big\|_2\lesssim\delta_t\max\{1,\|b_t^*\|_{\mathcal B^{(t)}}^2\}.\]
\end{theorem}

\begin{corollary}[RKHS cases with polynomial eigen decay]\label{cor:bridgeerrrkhs}
    We further suppose $\mathcal B^{(t)}$ and $\mathcal G^{(t)}$ are RKHSs for all $t=1,\dots, T$.  $K_{B,t}$ and $K_{G,t}$ are the kernels of $\mathcal B^{(t)}$ and $\mathcal G^{(t)}$ with non-increasing eigenvalues $\{\mu_{t,j}^{\mathcal B}\}_{j=1}^\infty$ and $\{\mu_{t,j}^{\mathcal G}\}_{j=1}^\infty$. We assume polynomial decay on eigenvalues, i.e., for some $\alpha_B,\alpha_G>\frac12$,
    \(\mu^{\mathcal B}_{t,j} \lesssim j^{-2\alpha_B},\quad
\mu^{\mathcal G}_{t,j} \lesssim j^{-2\alpha_G},\quad j\to\infty\).
Let $\alpha_{\min}:=\min\{\alpha_B,\alpha_G\}$.
Then the critical radius in \cref{thm:projbridgeerr} satisfy $\delta_{n_t}\lesssim \max\{\sqrt{U_t},L B_t\}n_t^{-\frac{\alpha_{\min}}{2\alpha_{\min}+1}} \log n_t$.
Consequently, under the conditions of Theorem~\ref{thm:projbridgeerr}, for all $t=1,\dots,T$, with probability at least $1-3\zeta$, 
\[
\|\mathcal T_t(\hat b_t-b_t^*)\|_2
\lesssim
\sqrt{\log(c_1/\zeta)}n_t^{-\frac{\alpha_{\min}}{2\alpha_{\min}+1}} \log n_t.
\]
\end{corollary}
\cref{thm:projbridgeerr} provides a finite-sample bound for the projected error
$\|\mathcal T_t(\hat b_t-b_t^*)\|_2$, where linear operator
$\mathcal T_t$ maps a bridge function to a conditional expectation given $(R_t,S_t,A_t)$.
However, for downstream analysis we also need control of the rooted mean-squared error (RMSE)
$\|\hat b_t-b_t^*\|_2$.

In general, converting projected error bounds into $L_2$ error bounds is nontrivial because
conditional moment problems are typically ill-posed inverse problems: the operator $\mathcal T_t$
is often compact and hence may not admit a stable inverse on an unrestricted function class.
This phenomenon and the role of regularization in such conditional moment models are well-studied
in the semi-/nonparametric literature; see, e.g., \cite{chen2011rate,chen2012estimation}.
We introduce an
ill-posedness measure for the conditional expectation operator $\mathcal T_t$,
following the definition in \citet{dikkala2020minimax}. Since the true bridge $b_t^*$ is not assumed to lie in $\mathcal B_B^{(t)}$, we define the best
approximation within the ball
\[
b_{t,*}:=\arg\min_{b\in\mathcal B_B^{(t)}}\|b-b_t^*\|_2,
\quad
\varepsilon_t(B_t):=\inf_{b\in\mathcal B_B^{(t)}}\|b-b_t^*\|_2.
\]
\begin{definition}[Measure of ill-posedness]\label{def:illposed}
Define the ill-posedness coefficient
\(
\tau_t(B_t)
:=
\sup_{b\in \mathcal B_B^{(t)}}
\frac{\|b-b_{t,*}\|_2}{\|\mathcal T_t(b-b_{t,*})\|_2},
\)
and assume $\tau_t(B_t)<\infty$.
\end{definition}
By combining \cref{thm:projbridgeerr} with \cref{def:illposed}, we obtain an $L_2$ error bound for the
bridge estimator:
\[\|\hat b_t-b_t^*\|_2 \le \tau_t(B_t)\delta_t +(\tau_t(B_t)+1) \varepsilon_t(B_t).\]
The choice of $B_t$ trades off the approximation bias $\varepsilon_t(B_t)$ and the ill-posedness
factor $\tau_t(B_t)$. Consider the whole function class $\mathcal B^{(t)}$, where $b_{t,*} = b_t^*$ and $\varepsilon_t(B_t) = 0$ under \cref{ass:realize}, this gives the global ill-posedness 
\[\tau_t
:=
\sup_{b\in \mathcal B^{(t)}}
\frac{\|b-b_{t,*}\|_2}{\|\mathcal T_t(b-b_{t,*})\|_2},\]
where we assume $\tau_t <\infty$. The RMSE is given by $\|\hat b_t-b_t^*\|_2\lesssim \tau_t\delta_t$.

\begin{theorem}[Bridge estimation error bound]\label{thm:bridgeerr}
For any $t=1,\dots, T$, suppose function class $\mathcal G^{(t)}$ is star-shaped and symmetric. Suppose $\mathcal G^{(t)}$ and $\mathcal B^{(t)}$ are $1$-uniformly bounded. Consider min-max estimator $\hat b_t$ estimated by \cref{eq:minimax}. Define function classes $\;\operatorname{star}\big(\mathcal B^{(t)}-b_{t}^*\big)
=
\{ r(b-b_{t}^*):\ b-b_{t}^*\in\mathcal B^{(t)}_B,\ r\in[0,1]\}$, and $\;
\operatorname{star}\big(\mathcal T_t(\mathcal B^{(t)}-b_{t}^*)\big)
=
\{ r\,g_{b,t}^U:\ b-b_{t}^*\in\mathcal B^{(t)}_B,\ r\in[0,1]\}$,
where $g^{U}_{b,t} = \arg\min_{g_t\in \mathcal G_U^{(t)}}\|g_t-\mathcal T_t(b-b_t^*)\|_2$. Define the $\delta_{n_t}$ as the upper bound on the critical radii of $\mathcal G_{3U}^{(t)}$ and the two function classes. Let $\delta_t = \delta_{n_t}+c_0\sqrt{\frac{\log(c_1/\zeta)}{n_t}}$.
Assume $\eta_t\lesssim \delta_t$, then if $\lambda \asymp \frac{\delta_t^2}{U}$ and $\mu\ge \frac43L^2+\frac{36}{B_t\lambda}\delta_t^2$, then with probability $1-3\zeta$, the following bound holds
\[\big\|\hat b_t-b_t^*\big\|_2\lesssim\tau_t\delta_t\max\{1,\|b_t^*\|_{\mathcal B^{(t)}}^2\}.\]
\end{theorem}

\subsection{Policy value estimation error bound}
Based on \cref{thm:bridgeerr}, we can further bound the OPE error of the policy value $\widehat V(\pi)$ estimated from $\cref{alg:prox-fqe-noncons}$. 
\begin{theorem}[Policy value estimation error bound]\label{thm:policyerr}
    Suppose RKHSs $\mathcal Q^{(t)}$, $\mathcal B^{(t)}$, $\mathcal G^{(t)}$ have polynomial eigen-decay rate
\[\mu^{\mathcal Q}_{t,j}\lesssim j^{-2\alpha_Q},
\mu^{\mathcal B}_{t,j}\lesssim j^{-2\alpha_B},
\mu^{\mathcal G}_{t,j}\lesssim j^{-2\alpha_G},\]
where $\alpha_Q, \alpha_B, \alpha_G>1/2$. Define $\alpha_{\min} = \min\{\alpha_Q, \alpha_B, \alpha_G\}>\frac12$. Denote $\delta_{t,*} = \bar \delta_{t,*} +c_0\sqrt{\frac{\log(c_1T/\zeta)}{n}}$ for some $c_0,c_1>0$ where $\bar \delta_{t,*}$ is the upper bound of the critical radii of difference classes $\Delta\mathcal Q^{(t)}$, $\Delta\mathcal Q^{(t+1)}$, $\Delta\mathcal B^{(t)}$ and $\mathcal G_U^{(t)}$ defined in \cref{app:proof-policyerr}. Suppose $\lambda_{Q,t}\asymp (\delta_{\Delta_\mathcal Q^{(t)}})^2$ and let $\tau_{\max} = \max_{t\le T}\tau_t$. Under \cref{ass:realize,ass:operatorbound,ass:boundconc,ass:concentrability} and assumptions for \cref{thm:identification}, with probability at least $1-\zeta$, the policy value estimation error is bounded by
    \begin{align*}
        \big|\widehat V(\pi) - V(\pi)\big|\lesssim K\tau_{\max}T^{2}\sqrt{\log(c_1 T/\zeta)}n^{-\frac{\alpha_{\min}}{2\alpha_{\min}+1}}\log n.
    \end{align*}
\end{theorem}
Without considering the ill-posedness, our OPE error bound achieves the optimal rate in $n$ in the classical nonparametric regression \cite{stone1982optimal}. Our error bound exhibits a $T^2$ dependence on the horizon, which arises from error propagation through the Bellman recursion and the additional complexity introduced by estimating bridge functions under the MNAR setting. 
For comparison, \citet{wang2024fine} provide a fine-grained analysis of FQE under fully observed rewards. Under the completeness assumption for $Q$-functions alone, they establish an error bound of order $\mathcal{O}(T^{1.5}\sqrt{1/n})$ for both parametric and nonparametric settings, improving upon the $\mathcal{O}(T^2\sqrt{\kappa/n})$ bounds in prior work \citep{duan2020minimax,zhang2022offline}. With an additional realizability assumption on the probability ratio functions $w_t^\pi$, the rate further improves to $\mathcal{O}(T\tilde{\kappa}\sqrt{1/n})$, matching the sharpest known bound under the tabular setting \citep{yin2020asymptotically}. The additional $T$ factor in our bound compared to the $T^{1.5}$ rate in \citet{wang2024fine} partially reflects the cost of correcting for MNAR rewards through the bridge function mechanism.

\section{Experiments}
In this section, we evaluate the performance of the proposed OPE estimator with rewards MNAR on both simulated and real-world data. We compare our method with four baselines: an IPW-based OPE method \cite{wang2025off}, a naive FQE baseline, an imputation-based OPE method and a reward-shaping-based OPE method \cite{parbhoo2020shaping}. The naive FQE estimator ignores the MNAR mechanism and applies FQE directly to the observed rewards. The IPW baseline adapts the weighting scheme of \citet{wang2025off}, which was originally developed for a trajectory dropout model. The imputation-based baseline learns the reward function by regression and imputes the unobserved reward using the fitted model. The reward-shaping baseline, named SCOPE, uses per-step importance sampling with potential-based reward shaping. 

We present results on simulation studies and a real-data application below. Our code is available at \url{https://github.com/NAIVlab/ShadOPE}.

\subsection{Simulation studies}
We conduct simulation studies in finite-horizon episodic MDPs to evaluate the proposed method under varying missingness levels and reward generation mechanisms. We compare ProxFQE against four baselines, naive FQE, IPW-FQE, Impute-FQE, and SCOPE, over 50 random seeds. 

We set state $S_t = (S_{t,1}, S_{t,2})^\top \in \mathcal S = \mathbb R^2$ as a two-dimensional vector. The action space is binary, $\mathcal A=\{-1,1\}$. Let $\mathcal{O}=\{0,1\}$ and $\mathcal{R}=\mathbb{R}$. The propensity score is set as 
\[e_t(S_t,A_t,R_t) = \operatorname{expit}\big(c_0 - 0.1A_t + 0.2(1,-2)^\top S_t +2.5 R_t \big)\] where $c_0$ is calibrated to attain target missing rates ranging from $20\%$ to $80\%$. The target policy we want to evaluate is given by 
\begin{align*}
    P_{\pi}(A_t=1 \mid S_t, O_{t-1}) =& \operatorname{expit}\big\{3[(1,0.3)^\top S_t +0.5\\&- 0.8 (2O_{t-1}-1)]\bigr\}
\end{align*}
See \cref{fig:data-overview-panel} in \cref{app:experiments} for visualization of the generated data. 

For function classes, we choose Gaussian kernels for $\mathcal{G}^{(t)}$ and $\mathcal{B}^{(t)}$. We report MSE under three settings: varying the sample size, varying the horizon length, and varying the reward generation mechanism. Specifically, we consider $n \in \{64,128,256,512,1024,2048\}$ with $T=8$, and $T \in \{2,4,8,16,32\}$ with $n=512$. More simulation details can be found in \cref{app:addisimsetup}.

\begin{figure*}[ht]
    \centering
    \includegraphics[width=0.98\textwidth]{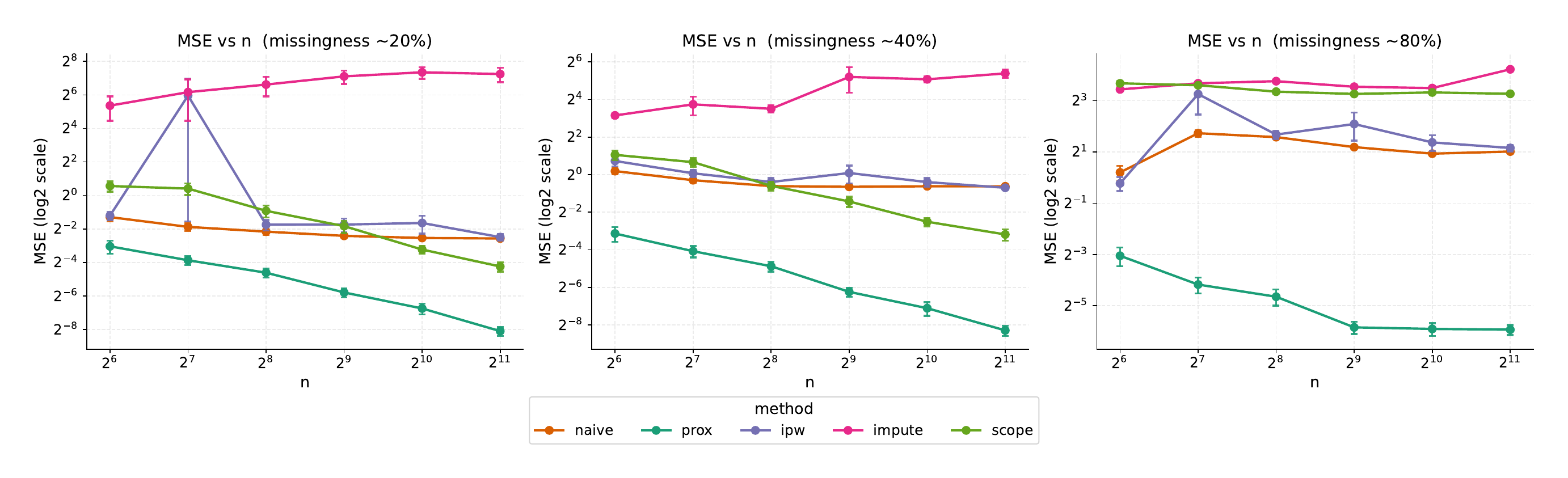}
    \caption{MSE vs.\ sample size ($n$) under three MNAR missingness percentages (${\sim}20\%$, ${\sim}40\%$, ${\sim}80\%$). \texttt{prox} consistently achieves the lowest MSE across all sample sizes and missingness levels, with MSE decreasing steadily as $n$ grows. }                
    \label{fig:mse-vs-n-missing}
\end{figure*}

\cref{fig:mse-vs-n-missing} shows the MSE of \texttt{prox} (our method), and four baselines \texttt{naive} (naive FQE), \texttt{ipw} (IPW-FQE), \texttt{impute} (Impute-FQE), and \texttt{scope} (SCOPE) under three MNAR missingness levels, approximately $20\%$, $40\%$, and $80\%$. Across all three panels, \texttt{prox} achieves the fastest error decay in $n$ and consistently the smallest MSE, with its advantage growing under heavier missingness. The baselines plateau at non-vanishing bias floors (\texttt{naive}, \texttt{ipw}), amplify selection bias (\texttt{impute}), or degrade under high missingness (\texttt{scope}), consistent with our theoretical analysis. Overall, the results demonstrate that \texttt{prox} is the most accurate and stable estimator in this setting.

Additional simulation results for varying horizon lengths and different reward generation mechanisms are reported in \cref{app:experiments} in the Supplemental Materials.

\subsection{Application: MIMIC-III Sepsis data}
We evaluate our method on a real-world clinical dataset from the MIMIC-III database \citep{johnson2016mimic}, using the sepsis cohort and pre-processing pipeline of \citet{raghu2017deep}. The dataset consists of ICU   
patients meeting the Sepsis-3 criteria, where physiological measurements, lab values, and treatment records are aggregated into 4-hour windows. Each patient's state at time $t$ is represented by a 48-dimensional feature vector comprising demographics and static indices (e.g., age, SOFA score, SIRS), laboratory values, vital sign, and intake/output variables. The action space consists of 25 discrete treatment options formed by 5 vasopressor dose levels $\times$ 5 IV fluid dose levels, where non-zero dosages are discretized into quartiles. 

We retain patients with at least 10 recorded time steps and truncate to the first 10, yielding a horizon of $T = 10$ with 13{,}943 patients. The reward at each step is defined as $R_t = -(SOFA_{t+1} - SOFA_t)$. Since the MIMIC-III rewards are fully observed, we introduce synthetic MNAR missingness to evaluate our method. The target policy is constructed by training a Double DQN on the fully-observed data and then applying a conservative dose reduction, yielding a missingness-aware policy $\pi(a \mid s_t, o_{t-1})$. Additional setups are provided in \cref{app:addirealdatasetup}.

We compare \texttt{prox} against \texttt{oracle} (oracle FQE), which uses fully observed rewards as a reference), \texttt{naive}, \texttt{impute}, \texttt{ipw}, and \texttt{scope}. 

\begin{figure*}[ht]
      \centering
      \includegraphics[width=0.7\textwidth]{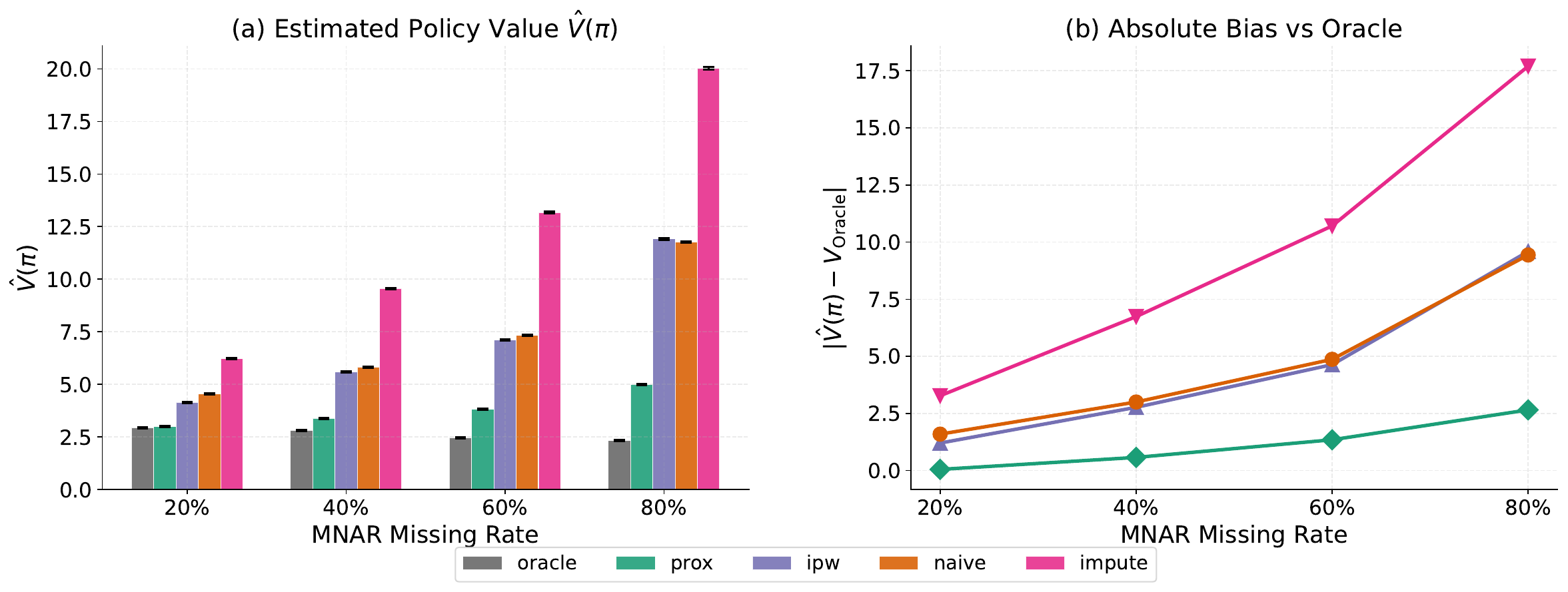}
      \caption{OPE results on MIMIC-III sepsis data under MNAR missing rates from 20\% to 80\%. (a) Estimated policy value $\hat{V}(\pi)$ with standard error bars. oracle FQE (gray) uses fully observed rewards as a reference. (b) Absolute bias relative to oracle FQE. SCOPE is excluded due to degenerate estimates.}            
      \label{fig:sepsis_ope}
  \end{figure*} 

\cref{fig:sepsis_ope} presents the results (full numerical results are provided in \cref{tab:sepsis_ope} in \cref{app:addirealdatasetup}). \texttt{prox} attains the lowest bias across all missing rates and remains close to \texttt{oracle} even under severe missingness, while all baselines exhibit substantially larger and growing bias as the missing rate increases. \texttt{scope} produces degenerate estimates due to near-zero importance weight overlap and is excluded from the figure.

\section{Conclusion and Discussion}
We study OPE in MDPs with MNAR rewards, proposing a bridge function approach that recovers the conditional mean reward without explicitly modeling the missingness mechanism. Unlike IPW-based methods \citep{wang2025off} that require external auxiliary variables, our method uses the next state as an endogenous shadow variable, avoiding additional data requirements and IPW variance inflation.

A limitation of our framework is that the reward missingness process may still be influenced by unobserved confounding factors that are not fully captured by the observed trajectories. While our identification relies on assumptions through the bridge function, violations of these assumptions may lead to biased policy value estimates in practice.

To address this concern, an important future direction is to incorporate sensitivity analysis into the proposed OPE framework. By parameterizing deviations from the bridge moment conditions or the completeness assumptions, one can assess the robustness of policy value estimates to potential unobserved confounding in the missingness mechanism. Such sensitivity analyses have been studied in proximal causal inference and POMDPs, and adapting their frameworks from proximal causal inference and POMDPs to the reward-missingness mechanism in MDPs is an important direction for future work.

\section*{Acknowledgement}
Qu's research is partially supported by NSF Grant DMS 2515275, NCI grant 1R01CA297869, NSF Grant CDS\&E-MSS 2401271. Miao's research is partially supported by Texas Artificial Intelligence Institute at University of Texas at Dallas.

\section*{Impact Statement}
This paper presents work whose goal is to advance the field of machine learning. There are many potential societal consequences of our work, none of which we feel must be specifically highlighted here.

\bibliography{references}

@article{miao2016varieties,
  title={On varieties of doubly robust estimators under missingness not at random with a shadow variable},
  author={Miao, Wang and Tchetgen Tchetgen, Eric J},
  journal={Biometrika},
  volume={103},
  number={2},
  pages={475--482},
  year={2016},
  publisher={Oxford University Press}
}

@article{miao2018identification,
  title={Identification and inference with nonignorable missing covariate data},
  author={Miao, Wang and Tchetgen, Eric Tchetgen},
  journal={Statistica Sinica},
  volume={28},
  number={4},
  pages={2049},
  year={2018}
}

@article{miao2022off,
  title={Off-policy evaluation for episodic partially observable markov decision processes under non-parametric models},
  author={Miao, Rui and Qi, Zhengling and Zhang, Xiaoke},
  journal={Advances in Neural Information Processing Systems},
  volume={35},
  pages={593--606},
  year={2022}
}

@book{sutton1998reinforcement,
  title={Reinforcement learning: An introduction},
  author={Sutton, Richard S and Barto, Andrew G and others},
  volume={1},
  number={1},
  year={1998},
  publisher={MIT press Cambridge}
}

@inproceedings{baird1995residual,
  title={Residual algorithms: Reinforcement learning with function approximation},
  author={Baird, Leemon and others},
  booktitle={Proceedings of the twelfth international conference on machine learning},
  pages={30--37},
  year={1995}
}

@article{dikkala2020minimax,
  title={Minimax estimation of conditional moment models},
  author={Dikkala, Nishanth and Lewis, Greg and Mackey, Lester and Syrgkanis, Vasilis},
  journal={Advances in Neural Information Processing Systems},
  volume={33},
  pages={12248--12262},
  year={2020}
}

@book{wainwright2019high,
  title={High-dimensional statistics: A non-asymptotic viewpoint},
  author={Wainwright, Martin J},
  volume={48},
  year={2019},
  publisher={Cambridge university press}
}

@book{kress1989linear,
  title={Linear integral equations},
  author={Kress, Rainer},
  volume={82},
  year={1989},
  publisher={Springer}
}

@article{fukumizu2009kernel,
  title={Kernel choice and classifiability for RKHS embeddings of probability distributions},
  author={Fukumizu, Kenji and Gretton, Arthur and Lanckriet, Gert and Sch{\"o}lkopf, Bernhard and Sriperumbudur, Bharath K},
  journal={Advances in neural information processing systems},
  volume={22},
  year={2009}
}

@article{foster2023orthogonal,
  title={Orthogonal statistical learning},
  author={Foster, Dylan J and Syrgkanis, Vasilis},
  journal={The Annals of Statistics},
  volume={51},
  number={3},
  pages={879--908},
  year={2023},
  publisher={Institute of Mathematical Statistics}
}

@article{bartlett2005local,
  title={Local rademacher complexities},
  author={Bartlett, Peter L and Bousquet, Olivier and Mendelson, Shahar},
  year={2005}
}

@inproceedings{munos2003error,
  title={Error bounds for approximate policy iteration},
  author={Munos, R{\'e}mi},
  booktitle={Proceedings of the Twentieth International Conference on International Conference on Machine Learning},
  pages={560--567},
  year={2003}
}

@article{munos2007performance,
  title={Performance bounds in l\_p-norm for approximate value iteration},
  author={Munos, R{\'e}mi},
  journal={SIAM journal on control and optimization},
  volume={46},
  number={2},
  pages={541--561},
  year={2007},
  publisher={SIAM}
}

@inproceedings{chen2019information,
  title={Information-theoretic considerations in batch reinforcement learning},
  author={Chen, Jinglin and Jiang, Nan},
  booktitle={International conference on machine learning},
  pages={1042--1051},
  year={2019},
  organization={PMLR}
}

@inproceedings{duan2021risk,
  title={Risk bounds and rademacher complexity in batch reinforcement learning},
  author={Duan, Yaqi and Jin, Chi and Li, Zhiyuan},
  booktitle={International Conference on Machine Learning},
  pages={2892--2902},
  year={2021},
  organization={PMLR}
}

@article{krieg2018tensor,
  title={Tensor power sequences and the approximation of tensor product operators},
  author={Krieg, David},
  journal={Journal of Complexity},
  volume={44},
  pages={30--51},
  year={2018},
  publisher={Elsevier}
}

@article{hong2023policy,
  title={A policy gradient method for confounded pomdps},
  author={Hong, Mao and Qi, Zhengling and Xu, Yanxun},
  journal={arXiv preprint arXiv:2305.17083},
  year={2023}
}

@article{li2025reinforcement,
  title={Reinforcement Learning with Continuous Actions Under Unmeasured Confounding},
  author={Li, Yuhan and Han, Eugene and Hu, Yifan and Zhou, Wenzhuo and Qi, Zhengling and Cui, Yifan and Zhu, Ruoqing},
  journal={Journal of the American Statistical Association},
  number={just-accepted},
  pages={1--26},
  year={2025},
  publisher={Taylor \& Francis}
}

@article{cui2024semiparametric,
  title={Semiparametric proximal causal inference},
  author={Cui, Yifan and Pu, Hongming and Shi, Xu and Miao, Wang and Tchetgen Tchetgen, Eric},
  journal={Journal of the American Statistical Association},
  volume={119},
  number={546},
  pages={1348--1359},
  year={2024},
  publisher={Taylor \& Francis}
}

@article{fischer2020sobolev,
  title={Sobolev norm learning rates for regularized least-squares algorithms},
  author={Fischer, Simon and Steinwart, Ingo},
  journal={Journal of Machine Learning Research},
  volume={21},
  number={205},
  pages={1--38},
  year={2020}
}

@article{chen2011rate,
  title={On rate optimality for ill-posed inverse problems in econometrics},
  author={Chen, Xiaohong and Reiss, Markus},
  journal={Econometric Theory},
  volume={27},
  number={3},
  pages={497--521},
  year={2011},
  publisher={Cambridge University Press}
}

@article{chen2012estimation,
  title={Estimation of nonparametric conditional moment models with possibly nonsmooth generalized residuals},
  author={Chen, Xiaohong and Pouzo, Demian},
  journal={Econometrica},
  volume={80},
  number={1},
  pages={277--321},
  year={2012},
  publisher={Wiley Online Library}
}

@article{uehara2023future,
  title={Future-dependent value-based off-policy evaluation in pomdps},
  author={Uehara, Masatoshi and Kiyohara, Haruka and Bennett, Andrew and Chernozhukov, Victor and Jiang, Nan and Kallus, Nathan and Shi, Chengchun and Sun, Wen},
  journal={Advances in neural information processing systems},
  volume={36},
  pages={15991--16008},
  year={2023}
}

@article{littman2001predictive,
  title={Predictive representations of state},
  author={Littman, Michael and Sutton, Richard S},
  journal={Advances in neural information processing systems},
  volume={14},
  year={2001}
}

@inproceedings{singh2003learning,
  title={Learning predictive state representations},
  author={Singh, Satinder P and Littman, Michael L and Jong, Nicholas K and Pardoe, David and Stone, Peter},
  booktitle={Proceedings of the 20th International Conference on Machine Learning (ICML-03)},
  pages={712--719},
  year={2003}
}

@inproceedings{xu2023instrumental,
  title={An instrumental variable approach to confounded off-policy evaluation},
  author={Xu, Yang and Zhu, Jin and Shi, Chengchun and Luo, Shikai and Song, Rui},
  booktitle={International Conference on Machine Learning},
  pages={38848--38880},
  year={2023},
  organization={PMLR}
}

@article{miao2015identification,
  title={Identification and doubly robust estimation of data missing not at random with an ancillary variable},
  author={Miao, Wang and Tchetgen Tchetgen, Eric and Geng, Zhi},
  year={2015},
  publisher={bepress}
}

@article{wang2024fine,
  title={A fine-grained analysis of fitted Q-evaluation: beyond parametric models},
  author={Wang, Jiayi and Qi, Zhengling and Wong, Raymond KW},
  journal={arXiv preprint arXiv:2406.10438},
  year={2024}
}

@inproceedings{parbhoo2020shaping,
  title={Shaping control variates for off-policy evaluation},
  author={Parbhoo, Sonali and Gottesman, Omer and Doshi-Velez, Finale},
  booktitle={Offline Reinforcement Learning Workshop at Neural Information Processing Systems (NeurIPS)},
  pages={93},
  year={2020}
}

@article{johnson2016mimic,
  title={MIMIC-III, a freely accessible critical care database},
  author={Johnson, Alistair EW and Pollard, Tom J and Shen, Lu and Lehman, Li-wei H and Feng, Mengling and Ghassemi, Mohammad and Moody, Benjamin and Szolovits, Peter and Anthony Celi, Leo and Mark, Roger G},
  journal={Scientific data},
  volume={3},
  number={1},
  pages={1--9},
  year={2016},
  publisher={Nature Publishing Group}
}

@article{raghu2017deep,
  title={Deep reinforcement learning for sepsis treatment},
  author={Raghu, Aniruddh and Komorowski, Matthieu and Ahmed, Imran and Celi, Leo and Szolovits, Peter and Ghassemi, Marzyeh},
  journal={arXiv preprint arXiv:1711.09602},
  year={2017}
}

@inproceedings{chu2023multiply,
  title={Multiply robust off-policy evaluation and learning under truncation by death},
  author={Chu, Jianing and Yang, Shu and Lu, Wenbin},
  booktitle={International Conference on Machine Learning},
  pages={6195--6227},
  year={2023},
  organization={PMLR}
}

@article{park2025evaluating,
  title={Evaluating and Learning Optimal Dynamic Treatment Regimes under Truncation by Death},
  author={Park, Sihyung and Lu, Wenbin and Yang, Shu},
  journal={arXiv preprint arXiv:2510.07501},
  year={2025}
}

@article{wang2025off,
  title={Off-Policy Evaluation Under Nonignorable Missing Data},
  author={Wang, Han and Xu, Yang and Lu, Wenbin and Song, Rui},
  journal={arXiv preprint arXiv:2507.06961},
  year={2025}
}

@article{ouyang2022training,
  title={Training language models to follow instructions with human feedback},
  author={Ouyang, Long and Wu, Jeffrey and Jiang, Xu and Almeida, Diogo and Wainwright, Carroll and Mishkin, Pamela and Zhang, Chong and Agarwal, Sandhini and Slama, Katarina and Ray, Alex and others},
  journal={Advances in Neural Information Processing Systems},
  volume={35},
  pages={27730--27744},
  year={2022}
}

@article{rafailov2023direct,
  title={Direct preference optimization: Your language model is secretly a reward model},
  author={Rafailov, Rafael and Sharma, Archit and Mitchell, Eric and Manning, Christopher D and Ermon, Stefano and Finn, Chelsea},
  journal={Advances in neural information processing systems},
  volume={36},
  pages={53728--53741},
  year={2023}
}

@article{achiam2023gpt,
  title={GPT-4 technical report},
  author={Achiam, Josh and Adler, Steven and Agarwal, Sandhini and Ahmad, Lama and Akkaya, Ilge and Aleman, Florencia Leoni and Almeida, Diogo and Altenschmidt, Janko and Altman, Sam and Anadkat, Shyamal and others},
  journal={arXiv preprint arXiv:2303.08774},
  year={2023}
}

@book{tsiatis2019dynamic,
  title={Dynamic Treatment Regimes: Statistical Methods for Precision Medicine},
  author={Tsiatis, Anastasios A and Davidian, Marie and Holloway, Shannon T and Laber, Eric B},
  year={2019},
  publisher={CRC press}
}

@article{murphy2003optimal,
  title={Optimal dynamic treatment regimes},
  author={Murphy, Susan A},
  journal={Journal of the Royal Statistical Society: Series B (Statistical Methodology)},
  volume={65},
  number={2},
  pages={331--355},
  year={2003},
  publisher={Wiley Online Library}
}

@article{mandel2014offline,
  title={Offline policy evaluation across representations with applications to educational games},
  author={Mandel, Travis and Liu, Yun-En and Levine, Sergey and Brunskill, Emma and Popovic, Zoran},
  journal={Proceedings of AAMAS},
  year={2014}
}

@article{levine2020offline,
  title={Offline reinforcement learning: Tutorial, review, and perspectives on open problems},
  author={Levine, Sergey and Kumar, Aviral and Tucker, George and Fu, Justin},
  journal={arXiv preprint arXiv:2005.01643},
  year={2020}
}

@article{uehara2022review,
  title={A Review of Off-Policy Evaluation in Reinforcement Learning},
  author={Uehara, Masatoshi and Shi, Chengchun and Kallus, Nathan},
  journal={arXiv preprint arXiv:2212.06355},
  year={2022}
}

@article{voloshin2019empirical,
  title={Empirical study of off-policy policy evaluation for reinforcement learning},
  author={Voloshin, Cameron and Le, Hoang M and Jiang, Nan and Yue, Yisong},
  journal={arXiv preprint arXiv:1911.06854},
  year={2019}
}

@article{little2019statistical,
  title={Statistical Analysis with Missing Data},
  author={Little, Roderick JA and Rubin, Donald B},
  journal={John Wiley \& Sons},
  edition={3rd},
  year={2019}
}

@article{kallus2020confounding,
  title={Confounding-robust policy evaluation in infinite-horizon reinforcement learning},
  author={Kallus, Nathan and Zhou, Angela},
  journal={Advances in Neural Information Processing Systems},
  volume={33},
  pages={22293--22304},
  year={2020}
}

@book{kaelbling1998planning,
  title={Planning and acting in partially observable stochastic domains},
  author={Kaelbling, Leslie Pack and Littman, Michael L and Cassandra, Anthony R},
  journal={Artificial Intelligence},
  volume={101},
  number={1-2},
  pages={99--134},
  year={1998},
  publisher={Elsevier}
}

@article{jaakkola1995reinforcement,
  title={Reinforcement learning algorithm for partially observable Markov decision problems},
  author={Jaakkola, Tommi and Singh, Satinder P and Jordan, Michael I},
  journal={Advances in Neural Information Processing Systems},
  pages={345--352},
  year={1995}
}

@article{tennenholtz2020off,
  title={Off-Policy Evaluation in Partially Observable Environments},
  author={Tennenholtz, Guy and Shalit, Uri and Mannor, Shie},
  journal={Proceedings of the AAAI Conference on Artificial Intelligence},
  volume={34},
  number={06},
  pages={10276--10283},
  year={2020}
}

@inproceedings{bennett2021off,
  title={Off-policy evaluation in infinite-horizon reinforcement learning with latent confounders},
  author={Bennett, Andrew and Kallus, Nathan and Li, Lihong and Mousavi, Ali},
  booktitle={International Conference on Artificial Intelligence and Statistics},
  pages={1999--2007},
  year={2021},
  organization={PMLR}
}

@inproceedings{shi2022minimax,
  title={A minimax learning approach to off-policy evaluation in confounded partially observable markov decision processes},
  author={Shi, Chengchun and Uehara, Masatoshi and Huang, Jiawei and Jiang, Nan},
  booktitle={International Conference on Machine Learning},
  pages={20057--20094},
  year={2022},
  organization={PMLR}
}

@inproceedings{ng1999policy,
  title={Policy invariance under reward transformations: Theory and application to reward shaping},
  author={Ng, Andrew Y and Harada, Daishi and Russell, Stuart},
  booktitle={International Conference on Machine Learning},
  pages={278--287},
  year={1999}
}

@article{devlin2012dynamic,
  title={Dynamic potential-based reward shaping},
  author={Devlin, Sam and Kudenko, Daniel},
  journal={Proceedings of the 11th International Conference on Autonomous Agents and Multiagent Systems},
  pages={433--440},
  year={2012}
}

@article{harutyunyan2015off,
  title={Off-policy reward shaping with ensembles},
  author={Harutyunyan, Anna and Brys, Tim and Vrancx, Peter and Now{\'e}, Ann},
  journal={arXiv preprint arXiv:1502.03248},
  year={2015}
}

@inproceedings{zhan2022offline,
  title={Offline reinforcement learning with realizability and single-policy concentrability},
  author={Zhan, Wenhao and Huang, Baihe and Huang, Audrey and Jiang, Nan and Lee, Jason D},
  booktitle={Conference on Learning Theory},
  pages={2730--2775},
  year={2022},
  organization={PMLR}
}

@inproceedings{rashidinejad2021bridging,
  title={Bridging offline reinforcement learning and imitation learning: A tale of pessimism},
  author={Rashidinejad, Paria and Zhu, Banghua and Ma, Cong and Jiao, Jiantao and Russell, Stuart},
  booktitle={Advances in Neural Information Processing Systems},
  volume={34},
  pages={11702--11716},
  year={2021}
}

@inproceedings{xie2021bellman,
  title={Bellman-consistent pessimism for offline reinforcement learning},
  author={Xie, Tengyang and Cheng, Ching-An and Jiang, Nan and Mineiro, Paul and Agarwal, Alekh},
  booktitle={Advances in Neural Information Processing Systems},
  volume={34},
  pages={6683--6694},
  year={2021}
}

@inproceedings{shi2023pessimistic,
  title={Pessimistic Q-learning for offline reinforcement learning: Towards optimal sample complexity},
  author={Shi, Laixi and Li, Gen and Wei, Yuting and Chen, Yuxin and Chi, Yuejie},
  booktitle={International Conference on Machine Learning},
  pages={31335--31385},
  year={2023},
  organization={PMLR}
}

@article{le2019batch,
  title={Batch policy learning under constraints},
  author={Le, Hoang M and Voloshin, Cameron and Yue, Yisong},
  journal={International Conference on Machine Learning},
  pages={3703--3712},
  year={2019},
  organization={PMLR}
}

@article{liu2018breaking,
  title={Breaking the curse of horizon: Infinite-horizon off-policy estimation},
  author={Liu, Qiang and Li, Lihong and Tang, Ziyang and Zhou, Dengyong},
  journal={Advances in Neural Information Processing Systems},
  volume={31},
  year={2018}
}

@inproceedings{kallus2020double,
  title={Double reinforcement learning for efficient off-policy evaluation in Markov decision processes},
  author={Kallus, Nathan and Uehara, Masatoshi},
  booktitle={International Conference on Machine Learning},
  pages={5078--5088},
  year={2020},
  organization={PMLR}
}

@article{miao2018identifying,
  title={Identifying causal effects with proxy variables of an unmeasured confounder},
  author={Miao, Wang and Geng, Zhi and Tchetgen Tchetgen, Eric J},
  journal={Biometrika},
  volume={105},
  number={4},
  pages={987--993},
  year={2018},
  publisher={Oxford University Press}
}

@article{tchetgen2020introduction,
  title={An introduction to proximal causal learning},
  author={Tchetgen Tchetgen, Eric J and Ying, Andrew and Cui, Yifan and Shi, Xu and Miao, Wang},
  journal={arXiv preprint arXiv:2009.10982},
  year={2020}
}

@article{bennett2021proximal,
  title={Proximal reinforcement learning: Efficient off-policy evaluation in partially observed markov decision processes},
  author={Bennett, Andrew and Kallus, Nathan},
  journal={arXiv preprint arXiv:2110.15332},
  year={2021}
}

@article{mohan2021graphical,
  title={Graphical models for processing missing data},
  author={Mohan, Karthika and Pearl, Judea},
  journal={Journal of the American Statistical Association},
  volume={116},
  number={534},
  pages={1023--1037},
  year={2021},
  publisher={Taylor \& Francis}
}

@article{sun2018semiparametric,
  title={Semiparametric estimation with data missing not at random using an instrumental variable},
  author={Sun, BaoLuo and Tchetgen Tchetgen, Eric J},
  journal={Statistica Sinica},
  volume={28},
  number={4},
  pages={1965},
  year={2018},
  publisher={NIH Public Access}
}

@article{shao2016semiparametric,
  title={Semiparametric inverse propensity weighting for nonignorable missing data},
  author={Shao, Jun and Wang, Lei},
  journal={Biometrika},
  volume={103},
  number={1},
  pages={175--187},
  year={2016},
  publisher={Oxford University Press}
}

@inproceedings{duan2020minimax,
  title        = {Minimax-Optimal Off-Policy Evaluation with Linear Function Approximation},
  author       = {Duan, Yaqi and Jia, Zeyu and Wang, Mengdi},
  booktitle    = {International Conference on Machine Learning},
  pages        = {2701--2709},
  year         = {2020},
  organization = {PMLR}
}

@inproceedings{zhang2022offline,
  title={Off-policy fitted q-evaluation with differentiable function approximators: Z-estimation and inference theory},
  author={Zhang, Ruiqi and Zhang, Xuezhou and Ni, Chengzhuo and Wang, Mengdi},
  booktitle={International Conference on Machine Learning},
  pages={26713--26749},
  year={2022},
  organization={PMLR}
}

@inproceedings{yin2020asymptotically,
  title        = {Asymptotically Efficient Off-Policy Evaluation for Tabular Reinforcement Learning},
  author       = {Yin, Ming and Wang, Yu-Xiang},
  booktitle    = {International Conference on Artificial Intelligence and Statistics},
  pages        = {3948--3958},
  year         = {2020},
  volume       = {108},
  series       = {Proceedings of Machine Learning Research},
  organization = {PMLR}
}

@article{stone1982optimal,
  title     = {Optimal global rates of convergence for nonparametric regression},
  author    = {Stone, Charles J},
  journal   = {The annals of statistics},
  pages     = {1040--1053},
  year      = {1982},
  publisher = {JSTOR}
}

@inproceedings{yang2018unbiased,
  title={Unbiased offline recommender evaluation for missing-not-at-random implicit feedback},
  author={Yang, Longqi and Cui, Yin and Xuan, Yuan and Wang, Chenyang and Belongie, Serge and Estrin, Deborah},
  booktitle={Proceedings of the 12th ACM conference on recommender systems},
  pages={279--287},
  year={2018}
}

@book{enders2022applied,
  title={Applied missing data analysis},
  author={Enders, Craig K},
  year={2022},
  publisher={Guilford Publications}
}

@article{zhang2024curses,
  title={On the curses of future and history in future-dependent value functions for off-policy evaluation},
  author={Zhang, Yuheng and Jiang, Nan},
  journal={Advances in Neural Information Processing Systems},
  volume={37},
  pages={124756--124790},
  year={2024}
}

@article{zhang2025statistical,
  title={Statistical tractability of off-policy evaluation of history-dependent policies in pomdps},
  author={Zhang, Yuheng and Jiang, Nan},
  journal={arXiv preprint arXiv:2503.01134},
  year={2025}
}

@inproceedings{majumdar2025concept,
  title={Concept-Based Off-Policy Evaluation},
  author={Majumdar, Ritam and Teversham, Jack and Parbhoo, Sonali},
  booktitle={Reinforcement Learning Conference},
  year={2025}
}

@article{kuang2026breaking,
  title={Breaking the Order Barrier: Off-Policy Evaluation for Confounded POMDPs},
  author={Kuang, Qi and Wang, Jiayi and Zhou, Fan and Qi, Zhengling},
  journal={Advances in Neural Information Processing Systems},
  volume={38},
  pages={49491--49538},
  year={2026}
}

@article{finn2016generalizing,
  title={Generalizing skills with semi-supervised reinforcement learning},
  author={Finn, Chelsea and Yu, Tianhe and Fu, Justin and Abbeel, Pieter and Levine, Sergey},
  journal={arXiv preprint arXiv:1612.00429},
  year={2016}
}

@inproceedings{li2023mahalo,
  title={Mahalo: Unifying offline reinforcement learning and imitation learning from observations},
  author={Li, Anqi and Boots, Byron and Cheng, Ching-An},
  booktitle={International Conference on Machine Learning},
  pages={19360--19384},
  year={2023},
  organization={PMLR}
}

@inproceedings{zheng2023semi,
  title={Semi-supervised offline reinforcement learning with action-free trajectories},
  author={Zheng, Qinqing and Henaff, Mikael and Amos, Brandon and Grover, Aditya},
  booktitle={International conference on machine learning},
  pages={42339--42362},
  year={2023},
  organization={PMLR}
}
\bibliographystyle{icml2026}

\newpage
\appendix
\onecolumn

\section{Related Work}

\paragraph{OPE.}
Off-policy evaluation has been extensively studied in the RL literature. Classical methods include IS and its variants \citep{liu2018breaking}, FQE \citep{le2019batch}, and doubly robust estimators that combine both \citep{kallus2020double}. Recent advances in offline RL have developed pessimistic approaches \citep{xie2021bellman,rashidinejad2021bridging,shi2023pessimistic,zhan2022offline} that achieve near-optimal sample complexity. For comprehensive reviews, see \citet{uehara2022review} and \citet{levine2020offline}.

OPE in POMDPs has received growing attention, with works addressing latent confounding \citep{bennett2021off,kallus2020confounding}, partial observability \citep{tennenholtz2020off,shi2022minimax,miao2022off}, and future-dependent estimation \citep{uehara2023future,zhang2024curses}. More recent extensions study OPE of history-dependent policies through model-based methods \citep{zhang2025statistical}, confounded POMDPs \citep{kuang2026breaking}, and concept-based representations \citep{majumdar2025concept}. These works tackle partial observability in the state process, while our work addresses a complementary challenge: MNAR missingness in the reward process under a fully observed MDP.

\paragraph{Missing Data.}
Missing data problems have been extensively studied in statistics \citep{little2019statistical,enders2022applied}. Under missing at random (MAR) assumptions, inverse probability weighting and doubly robust methods are well-established. For MNAR, identification typically requires additional structure such as instrumental variables \citep{sun2018semiparametric}, shadow variables \citep{shao2016semiparametric,miao2016varieties,miao2018identifying}, or graphical constraints \citep{mohan2021graphical}. Proximal causal inference \citep{tchetgen2020introduction,bennett2021proximal,cui2024semiparametric} has emerged as a powerful framework for handling unmeasured confounding using proxy variables. Our work extends these ideas to the OPE setting with MNAR rewards, leveraging future states as shadow variables for identification.

\paragraph{Reward Shaping.} Potential-based reward shaping augments rewards with a potential difference to densify sparse feedback while preserving the optimal policy \citep{ng1999policy}, with later extensions to dynamic shaping and off-policy settings \citep{devlin2012dynamic,harutyunyan2015off,parbhoo2020shaping}. These methods are designed to accelerate policy learning or reduce variance under fully observed rewards, whereas our work targets identification under MNAR rewards via a bridge function.

\paragraph{Semi-supervised RL.}
Our setting also relates to semi-supervised RL and partially reward-labeled sequential decision-making, including semi-supervised RL with rewards available only in labeled MDPs \citep{finn2016generalizing}, semi-supervised offline RL with mixed labeled and unlabeled trajectories \citep{zheng2023semi}, and offline policy learning with partially reward-labeled trajectories \citep{li2023mahalo}. These works focus on policy learning or reward modeling under label scarcity, whereas we target off-policy \emph{evaluation} under an MNAR missingness mechanism that depends on the latent reward.

\section{Discussion on shadow variables}
In the causal inference literature on missing data, identification typically requires additional assumptions, most commonly instantiated through either instrumental variables or shadow variables. Shadow-variable approaches \cite{miao2015identification,miao2016varieties} leverage a fully observed variable that is informative about the outcome while being conditionally independent of the missingness mechanism given covariates and the (possibly unobserved) outcome. In contrast, instrumental-variable approaches \cite{sun2018semiparametric} posit a variable that shifts the missingness mechanism but has no direct effect on the outcome. 

In sequential decision making, \citet{wang2025off} develop an approach that requires specifying a stage-wise shadow variable $Z_t$ that satisfies conditional independence with dropout given $(S_t,A_t,R_{t+1},S_{t+1})$ while remaining informative about $(R_{t+1},S_{t+1})$ on the observed subset, effectively providing the identifying leverage needed to learn the dropout model.

While such a choice can be plausible, it may rely on additional measurements or domain knowledge to select a valid $Z_t$. In contrast, we adopt an endogenous choice of shadow variable, which is the next state $S_{t+1}$. Under the exclusion restriction and relevance condition in \cref{ass:exclusion,ass:relevance}, $S_{t+1}$ provides a readily available proxy for the MNAR reward without introducing an extra auxiliary variable.  

Moreover, a natural extension of $S_{t+1}$ is a multi-step future variable, or a low-dimensional summary thereof. This aligns with a broader predictive-state perspective in POMDPs, where future observations are used to encode information about the latent state \citep{littman2001predictive,singh2003learning}. More recently, \citet{xu2023instrumental,uehara2023future} use multi-step futures to stand in for the unobserved state: instead of conditioning on the latent state, they condition on a future window and learn quantities from it, using the future window as a proxy that carries latent-state information.

In missing data problems in MDPs, using longer futures can carry richer information about the missing rewards, and may make relevance and completeness-type conditions more plausible, but it also increases the statistical and computational burden as the future window grows. In practice, these tradeoffs motivate using compact summaries of multi-step futures.

\section{Additional Experiment Details} \label{app:experiments}
\subsection{Additional simulation setups}
\label{app:addisimsetup}
We set the behavior policy that generates the offline trajectories as
\[
P_{\pi^b}(A_t = 1 \mid S_t) = \operatorname{expit}\!\big(0.3 + (0.8,-0.3)^\top S_t\big).
\]
The next state $S_{t+1}$ is generated by transition kernel $S_{t+1} = 0.9 S_t + 0.2A_t\mathbf 1_2 +\mathcal N(0, 0.1^2 I_2)$ where $\mathbf 1_2 = (1,1)^\top$ and initial state \(S_1 \sim \mathcal N(0, I_2)\). We consider two reward generation mechanisms. The first is the sigmoid reward model used in other simulations:
\begin{align*}
   R_t = \operatorname{expit}([0.9-0.6A_t,-0.7]^\top S_t + [1.3,2]^\top S_{t+1} - 0.4A_t) + U_t,
\quad U_t \sim \mathrm{Unif}[-0.1,0.1], 
\end{align*}
and the second is a linear reward model:
\begin{align*}
 R_t = \mathrm{clip}([0.5,-0.3]^\top S_t + [0.8,0.6]^\top S_{t+1} - 0.3A_t + \epsilon_t,\,-1,\,1),
\quad \epsilon_t \sim N(0,0.01).   
\end{align*}

The true policy value is estimated by Monte Carlo using $5000$ independent trajectories generated under the target policy. We visualize the generated data for the setting $n=1000$, $T=10$, and random seed $44$; see \cref{fig:data-overview-panel} for an overview.

For all the RKHSs, the bandwidths are selected by median heuristic trick \cite{fukumizu2009kernel}; parameter $\delta$ is set to $\delta_t = 5n_t^{-0.4}$ according to \cite{dikkala2020minimax}. The penalty parameter $\lambda_{\mathrm{rkhs}}$ is chosen by cross-validation.

\begin{figure}[ht]
  \vskip 0.2in
  \begin{center}
    \centerline{\includegraphics[width=\columnwidth]{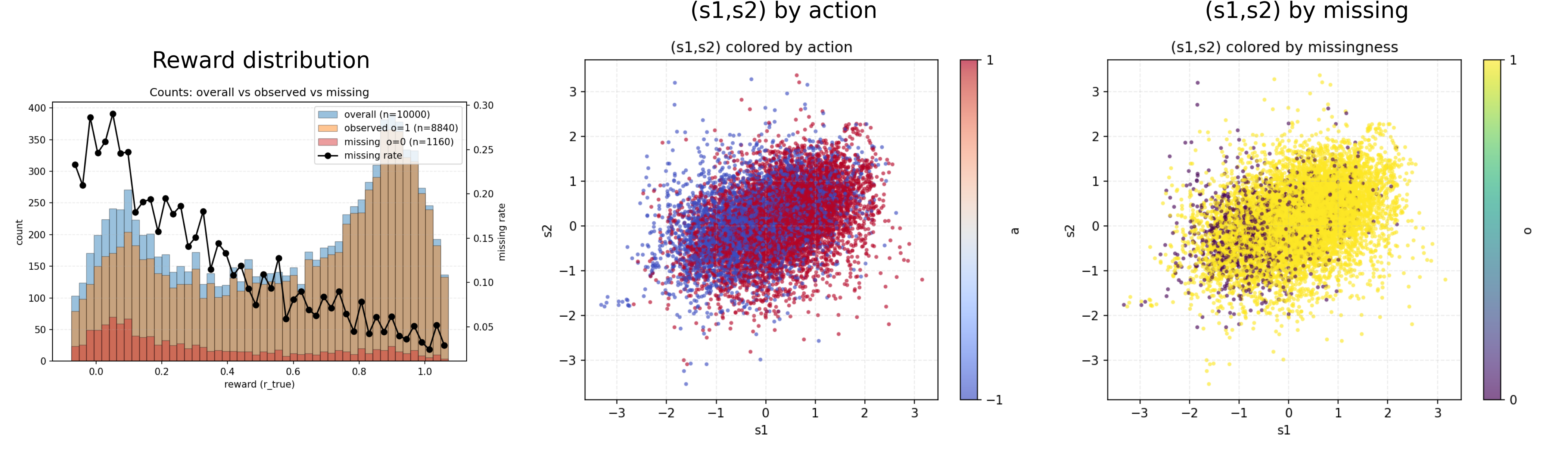}}
    \caption{\textbf{Data overview.}
    \textbf{Left}: histogram of the true reward \(r_{\text{true}}\) with three overlays: overall (blue), observed \(O{=}1\) (orange), and missing \(O{=}0\) (red). The total missing rate is $11.60\%$. Missing mass is relatively larger in the low–reward region.
    \textbf{Middle}: state value \(S_t = (s_1,s_2)\) colored by action (blue: \(a{=}-1\), red: \(a{=}+1\)) according to target policy, showing both actions across the state space without obvious coverage gaps.
    \textbf{Right}: state value \(S_t = (s_1,s_2)\) colored by missingness (yellow: \(O{=}1\), purple: \(O{=}0\)); the non-uniform placement of missing points indicates observation probability varies with state.}
    \label{fig:data-overview-panel}
  \end{center}
\end{figure}

\subsection{Baseline methods for simulation}
\subsubsection{Naive FQE}
In naive FQE baseline, we ignore the missingness mechanism, and perform FQE only on observed samples:
\[\widehat Q_t^{\mathrm{naive}}
=
\arg\min_{Q\in\mathcal Q^{(t)}}
\frac{1}{n_t}\sum_{i\in\mathcal I_t^{\mathrm{obs}}}\Big(Q(S_{t,i},A_{t,i})-y_{t,i}^{\mathrm{naive}}\Big)^2
+\lambda_{\rm rkhs}\|Q\|_{\mathcal Q^{(t)}}^2,\]
where 
\[y_{t,i}^{\mathrm{naive}}
=
\begin{cases}
R_{t,i}, & t=T,\\
R_{t,i}+\sum_{a\in\mathcal A}\pi(a\mid S_{t+1,i},O_{t,i})\,\widehat Q_{t+1}^{\mathrm{naive}}(S_{t+1,i},a), & t<T.
\end{cases}\]

As discussed in \cref{sec:identification}, under reward MNAR we generally have
\[
\mathbb E[R_t \mid S_t=s, A_t=a, O_t=1]\neq \mathbb E[R_t \mid S_t=s, A_t=a]
\quad \text{for some }(s,a),
\]
so naive FQE, which regresses on observed rewards, can be biased.

\subsubsection{IPW-FQE}
We adapt the IPW-based FQE method of \citet{wang2025off} to our reward MNAR setting. Since the true reward $R_t$ is unobserved when $O_t=0$, we cannot directly condition the propensity score on $R_t$. Instead, we first fit an RKHS bridge function $\hat b_t$ on observed samples $(O_t=1)$ to obtain $\widetilde R_t$, and then estimate the extended propensity score via logistic regression on features $(S_t, A_t, \hat b_t)$:                             
\[\hat e_t(S_t, A_t, \hat b_t) = \hat P(O_t = 1 \mid S_t, A_t, \hat b_t).\]                                          At each stage $t$, IPW-FQE solves a weighted kernel ridge regression on the observed subset $\mathcal I_t^{\mathrm{obs}}$:
\begin{align*}
    \widehat Q_t^{\mathrm{ipw}}
=
\arg\min_{Q\in\mathcal Q^{(t)}}
\frac{1}{n_t}\sum_{i\in \mathcal I_t^{\mathrm{obs}}} w_{t,i}\Big(Q(S_{t,i},A_{t,i})-y_{t,i}^{\mathrm{ipw}}\Big)^2 + \lambda_{\rm rkhs}\|Q\|_{\mathcal Q^{(t)}}^2,
\end{align*}     
  where $w_{t,i}=\frac{1}{\hat e_{t,i}}$ and 
  \[y_{t,i}^{\mathrm{ipw}}           
  =
  \begin{cases}                          
  R_{t,i}, & t=T,\\
  R_{t,i}+\sum_{a\in\mathcal A}\pi(a\mid S_{t+1,i},O_{t,i})\,\widehat Q_{t+1}^{\mathrm{ipw}}(S_{t+1,i},a), & t<T.                                    
  \end{cases}\]     
  
\subsubsection{Impute-FQE}
At each step $t$, we fit a kernel ridge regression on the observed subset to learn the conditional mean of the reward function $m_t(S_t, A_t)$. The imputed rewards are then constructed as
  \[              
  \widetilde R_{t,i}^{\mathrm{imp}} = O_{t,i} \cdot R_{t,i} + (1 - O_{t,i}) \cdot \hat m_t(S_{t,i}, A_{t,i}),
  \]
and FQE proceeds on all $n_t$ samples using $\widetilde R_{t,i}^{\mathrm{imp}}$ in place of $R_{t,i}$:
  \[\widehat Q_t^{\mathrm{imp}} = \arg\min_{Q \in \mathcal{Q}^{(t)}} \frac{1}{n} \sum_{i=1}^{n} \Big(Q(S_{t,i}, A_{t,i}) - y_{t,i}^{\mathrm{imp}}\Big)^2
  + \lambda_{\mathrm{rkhs}} \|Q\|_{\mathcal{Q}^{(t)}}^2,\]
  where
  \[y_{t,i}^{\mathrm{imp}} = \begin{cases}
  \widetilde R_{t,i}^{\mathrm{imp}}, & t = T, \\                              \widetilde R_{t,i}^{\mathrm{imp}} + \sum_{a \in \mathcal{A}} \pi(a \mid S_{t+1,i}, O_{t,i})\, \widehat Q_{t+1}^{\mathrm{imp}}(S_{t+1,i}, a), & t < T.
  \end{cases}\]
  Since $\hat m_t$ is trained only on the observed subset $\{i : O_{t,i} = 1\}$, it estimates $\mathbb{E}[R_t \mid S_t, A_t, O_t = 1]$ rather than         
  $\mathbb{E}[R_t \mid S_t, A_t]$. Under MNAR, the imputed values inherit the selection bias from the observed subsample, leading to biased Q-function     
  estimates.

\subsubsection{SCOPE}
SCOPE \citep{parbhoo2020shaping} is a per-step importance sampling estimator designed for sparse reward settings, which incorporates potential-based reward shaping \cite{ng1999policy} as a control variate to reduce variance. To apply SCOPE in our MNAR setting, we replace unobserved rewards with zero. This introduces bias because under MNAR, missingness is informative and zero-imputation does not recover the true expected reward.

\subsection{Other simulation results}
\cref{fig:mse_vs_T} reports MSE as the horizon $T$ varies from $2$ to $32$. Error compounding in backward induction affects all methods, but the growth
  rate differs markedly. \texttt{prox} remains below MSE $\approx 1$ for $T \le 16$ across all missingness levels, whereas \texttt{ipw} already exceeds 
  $10^2$ at $T=16$ under ${\sim}40\%$ missingness. \texttt{impute} is the worst-performing method at longer horizons, with MSE exceeding $10^8$ at $T=32$ under ${\sim}20\%$ missingness, because imputation errors at each step feed into subsequent Q-function fits and amplify exponentially. \texttt{scope} is the most competitive baseline for short horizons ($T\le 4$) but its per-step IS weights accumulate variance as $T$ grows. These results highlight that the bridge-function approach in \texttt{prox} is particularly advantageous in long-horizon problems.

  \begin{figure}[ht]
    \centering
    \includegraphics[width=\textwidth]{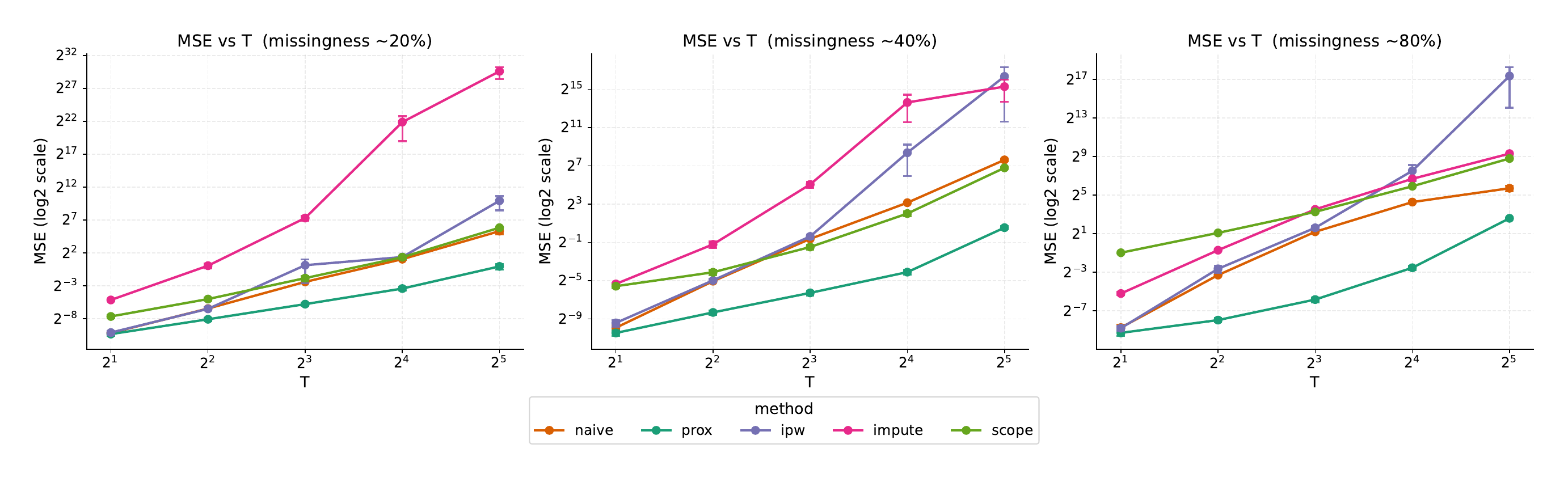}
    \caption{MSE vs.\ horizon length ($T$) under three MNAR missingness percentages (${\sim}20\%$, ${\sim}40\%$, ${\sim}80\%$). \texttt{prox} scales most gracefully as $T$ grows, maintaining orders-of-magnitude lower MSE than all baselines. \texttt{ipw} and \texttt{impute} explode at large $T$ due to variance blow-up and compounding imputation errors, respectively. \texttt{naive} grows steadily as MNAR bias compounds at each backward step.            
  \texttt{scope} remains more stable than \texttt{ipw} but degrades sharply under high missingness.} \label{fig:mse_vs_T}
  \end{figure}

For the reward type comparison, we fix $n=512$ and $T=8$ and evaluate under the two reward generation mechanisms: a bounded sigmoid reward and a clipped linear reward. \cref{fig:mse_reward} compares MSE across the two reward  generation mechanisms. Under the sigmoid reward, \texttt{prox} achieves MSE on the order of $10^{-2}$, while the best baseline (\texttt{naive}) is above $10^{-1}$ even at ${\sim}20\%$ missingness. The nonlinearity of the sigmoid function makes regression-based imputation particularly difficult, explaining the poor performance of \texttt{impute} on this reward type. Under the linear reward, all methods improve and the gap between \texttt{prox} and baselines narrows. Notably, \texttt{impute} and \texttt{scope} become competitive with \texttt{prox} at ${\sim}80\%$ missingness for the linear reward, suggesting that simpler reward structures are more amenable to naive correction strategies.

  \begin{figure}[ht]
    \centering
    \includegraphics[width=\textwidth]{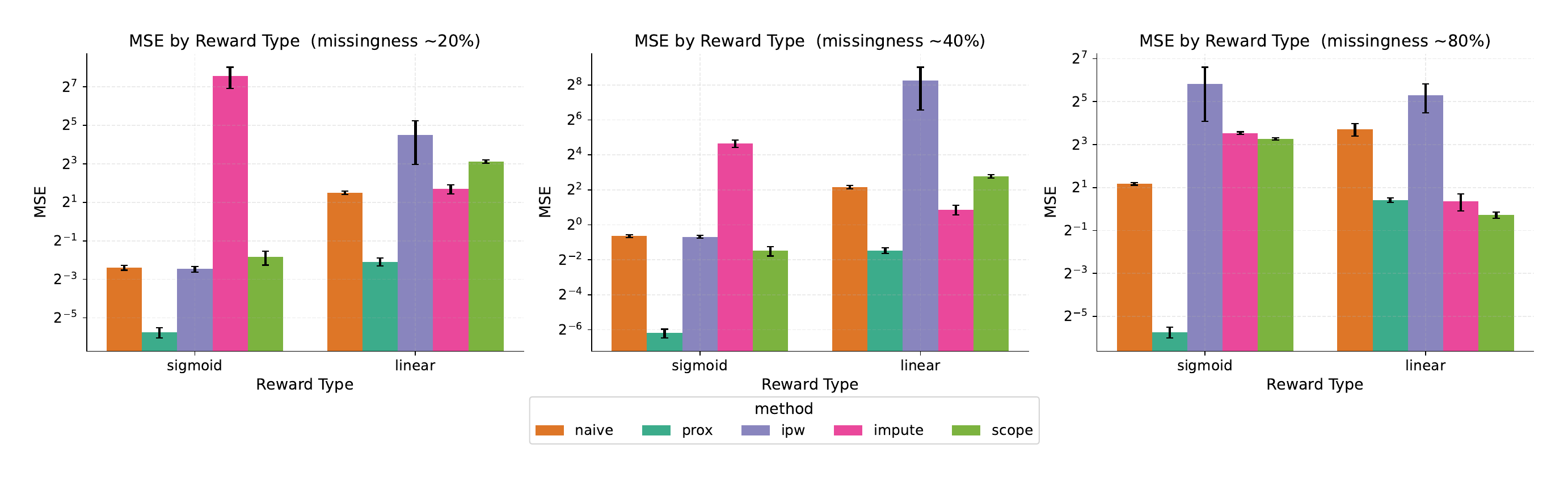}
    \caption{MSE by reward type (sigmoid vs.\ linear) under three MNAR missingness percentages (${\sim}20\%$, ${\sim}40\%$, ${\sim}80\%$). For the sigmoid 
  reward, \texttt{prox} achieves MSE orders of magnitude lower than all baselines. For the linear reward, the gap narrows but \texttt{prox} still leads.   
  \texttt{impute} performs comparably on the linear reward but poorly on sigmoid. \texttt{scope} degrades sharply under high missingness for sigmoid.}     
    \label{fig:mse_reward}                                 \end{figure}

\subsection{Additional real data experiment setups}
\label{app:addirealdatasetup}
The propensity score is set as
  \begin{align*}
      e_t(S_t,A_t,R_t) =  \operatorname{expit}(c_0 + 2.0 \cdot R_t^{\mathrm{std}} + 0.3 \cdot S_t^{\mathrm{signal}}- 0.1 \cdot A_t^{\mathrm{norm}}), 
  \end{align*}
where $R_t^{\mathrm{std}} = (R_t - \bar{R}) / \mathrm{sd}(R)$ is the standardized reward, $A_t^{\mathrm{norm}}$ is the standardized total treatment intensity (sum of vasopressor and IV fluid dose levels, then z-scored), and $S_t^{\mathrm{signal}}$ is a clinical severity signal defined as the normalized arterial lactate minus half the normalized mean blood pressure. The positive coefficient on the reward ($c_r = 2.0$) ensures that higher rewards are more likely to be observed while lower rewards tend to be missing, reflecting a clinically plausible scenario where SOFA-based outcomes depend on lab measurements that may be ordered less frequently for deteriorating patients. The intercept $c_0$ is calibrated via bisection to achieve target missing rates of 20\%, 40\%, 60\%, and 80\%. 

The target policy is constructed in two stages. First, we train a Double DQN with two hidden layers of size 128 on the original fully-observed data using the 48-dimensional state and $\gamma = 1.0$. This produces a greedy base policy $\pi_{\mathrm{DQN}}(s_t)$ over the 25 actions. Second, we construct a missingness-aware policy $\pi(a \mid s_t, o_{t-1})$ by applying a conservative dose reduction when the previous reward is unobserved: if $O_{t-1} = 0$, both the vasopressor and IV fluid dose levels recommended by the DQN are reduced by one (clipped at zero). This adjustment reflects a clinically natural response---reducing treatment intensity when the previous outcome is unknown---and produces a target policy whose action distribution depends on $O_{t-1}$, consistent with the policy class in our framework.

We split patients 60/40 by ID for model fitting and evaluation, respectively. On the fitting set, we train all OPE models including bridge functions, Q-networks, and auxiliary models for the baselines. Due to the high-dimensional state space and large action space, all methods use neural network function approximation, with architectures of $(512, 512, 256)$ for Q-networks and bridge networks and $(256, 256)$ for auxiliary networks (imputation regressor, propensity classifier). On the held-out 40\% test set, we compute the per-patient estimated value $\hat{V}_i(\pi)$ and report the mean and standard error. See detailed experiment results in \cref{tab:sepsis_ope}.

\begin{table}[h]
  \centering
  \small
  \setlength{\tabcolsep}{4pt}
  \caption{OPE results on MIMIC-III sepsis data. We report the estimated policy value $\hat{V}(\pi)$ with standard errors and absolute bias relative to oracle FQE.}
  \label{tab:sepsis_ope}
  \begin{tabular}{lcccccccc}
  \toprule
  & \multicolumn{2}{c}{20\% MNAR} & \multicolumn{2}{c}{40\% MNAR} & \multicolumn{2}{c}{60\% MNAR} & \multicolumn{2}{c}{80\% MNAR} \\
  \cmidrule(lr){2-3} \cmidrule(lr){4-5} \cmidrule(lr){6-7} \cmidrule(lr){8-9}
  Method & $\hat{V}(\pi)$ & Bias & $\hat{V}(\pi)$ & Bias & $\hat{V}(\pi)$ & Bias & $\hat{V}(\pi)$ & Bias \\
  \midrule
  oracle  & 2.94 $\pm$ 0.03 & --- & 2.81 $\pm$ 0.03 & --- & 2.46 $\pm$ 0.03 & --- & 2.32 $\pm$ 0.03 & --- \\
  Prox    & 2.99 $\pm$ 0.03 & 0.05 & 3.39 $\pm$ 0.03 & 0.58 & 3.81 $\pm$ 0.03 & 1.35 & 4.98 $\pm$ 0.03 & 2.66 \\
  IPW     & 4.15 $\pm$ 0.03 & 1.20 & 5.58 $\pm$ 0.03 & 2.77 & 7.10 $\pm$ 0.03 & 4.64 & 11.91 $\pm$ 0.04 & 9.59 \\
  naive   & 4.54 $\pm$ 0.03 & 1.59 & 5.81 $\pm$ 0.03 & 3.00 & 7.33 $\pm$ 0.03 & 4.87 & 11.76 $\pm$ 0.03 & 9.44 \\
  impute  & 6.22 $\pm$ 0.03 & 3.27 & 9.55 $\pm$ 0.04 & 6.74 & 13.17 $\pm$ 0.04 & 10.71 & 20.02 $\pm$ 0.06 & 17.70 \\
  scope   & 10408.2 $\pm$ 1991.0 & 10405.3 & $-$4089.2 $\pm$ 786.8 & 4092.0 & 10667.9 $\pm$ 2037.3 & 10665.4 & $-$7096.6 $\pm$ 1373.6 & 7099.0 \\
  \bottomrule
  \end{tabular}
  \end{table}

\section{Proof of \cref{thm:identification}}\label{app:proof-ident}
In this section we provide the proof of identification result in \cref{thm:identification}. 

\subsection*{Part A: Identification by bridge functions}
We first show that the policy value can be identified by the bridge functions $\{b_t\}_{t=1}^T$ if the bridges exist. 

Fix $t\in\{1,\dots,T\}$ and suppose there exists a measurable function
$b_t:\mathcal S\times\mathcal A\times\mathcal S\to\mathbb R$ satisfying the \cref{eq:bridge}. By \cref{ass:nofuture,ass:positivity}, conditioning on the event $O_t=1$ is well-defined and 
\[
\mathbb E[b_t(S_t,A_t,S_{t+1})\mid R_t,S_t,A_t,O_t=1]
=\mathbb E[b_t(S_t,A_t,S_{t+1})\mid R_t,S_t,A_t].
\]
So \cref{eq:bridge} holds if and only if \cref{eq:popbridge} holds. 

For the imputed reward $\widetilde R_t$ defined by \cref{eq:impreward}, \cref{eq:unbiased} implies that it has the same conditional mean reward as $R_t$ given $(S_t,A_t)$:
\[\mathbb E[\widetilde R_t\mid S_t = s, A_t = a] = \mathbb E[b_t(s,a,S_{t+1})\mid S_t=s,A_t=a] = \mathbb E[R_t\mid S_t=s,A_t=a]:= \bar r_t(s,a).\]

Now consider the augmented process. By \cref{ass:markov}, the augmented process is Markov, and the Bellman recursion holds with one-step reward $\bar r_t(S_t,A_t)$:
\[
\begin{aligned}
Q_t^\pi(s,a)
&=\mathbb E\!\big[\bar r_t(s,a)+V_{t+1}^\pi(S_{t+1},O_t)\mid S_t=s,A_t=a\big],\\
V_t^\pi(s,o_-)
&=\sum_a \pi_t(a\mid s,o_-)Q_t^\pi(s,a),\qquad V_{T+1}^\pi\equiv 0,
\end{aligned}
\]
which is equivalent to \cref{eq:bellman}. Therefore, the policy value $V(\pi)=\mathbb E[V_1^\pi(\widetilde S_1)]$ is identified.

\subsection*{Part B: Existence of bridge functions}
We now establish existence of the bridges. For a probability measure $\mu$, let $\mathcal L^2(\mu)$ denote the space of all squared integrable functions of $x$ with respect to measure $\mu(x)$, which is a Hilbert space endowed with the inner product $\langle g_1, g_2\rangle = \int g_1(x)g_2(x)d\mu(x)$. For any $s, a, t$, we define operator 
\[\mathcal T_{t\mid(s,a)}:\mathcal L^2\big(P_{S_{t+1}\mid s,a}\big)\to \mathcal L^2\!\big(P_{R_t\mid s,a}\big),\]
where $(\mathcal T_{t\mid(s,a)}h)(r):=\mathbb E\!\big[h(S_{t+1})\mid R_t=r,S_t=s,A_t=a\big]$. Its adjoint operator is defined by
\[\mathcal T_{t\mid(s,a)}^*:\mathcal L^2\big(P_{R_t\mid s,a}\big)\to \mathcal L^2\!\big(P_{S_{t+1}\mid s,a}\big),\]
where
$(\mathcal T_{t\mid(s,a)}^*g)(s')=\mathbb E\big[g(R_t)\mid S_{t+1}=s',S_t=s,A_t=a\big]$. Then, the bridge equation \eqref{eq:bridge} can be written as a first-kind Fredholm integral equation
\[
(\mathcal T_{t\mid(s,a)}h)(r)=g(r),
\]
where the unknown functions are $h(\cdot)=b_t(s,a,\cdot)\in \mathcal L^2\big(P_{S_{t+1}\mid s,a}\big)$ and the right hand side is $g(r)=r\in\mathcal L^2(P_{R_t\mid s,a})$.

\begin{assumption}[Hilbert-Schmidt property]\label{ass:integral}
    For any $(s,a)\in \mathcal {S\times A}$, for all $t = 1,\dots, T$, denote conditional densities $p_{S_{t+1}\mid R_t}(s'\mid r,s,a),\;\;  p_{R_t\mid S_{t+1}}(r\mid s',s,a)$. We have
\[\int_{\mathcal R}\int_{\mathcal S} p_{S_{t+1}\mid R_t}(s'\mid r,s,a) p_{R_t\mid S_{t+1}}(r\mid s',s,a) ds'dr<\infty.\]
\end{assumption}
This ensures that the operator $\mathcal T_{t\mid(s,a)}$ is Hilbert–Schmidt and thus admits a singular system
$\{(\sigma_{s,a,t,\nu},\varphi_{s,a,t,\nu},\psi_{s,a,t,\nu})\}_{\nu\ge1}$ satisfying
$\mathcal T_{t\mid(s,a)}\varphi_{s,a,t,\nu}=\sigma_{s,a,t,\nu}\psi_{s,a,t,\nu}$ and
$\mathcal T_{t\mid(s,a)}^\ast\psi_{s,a,t,\nu}=\sigma_{s,a,t,\nu}\varphi_{s,a,t,\nu}$.

\begin{assumption}\label{ass:sumfinite}
    Suppose $\{(\sigma_{s,a,t,\nu},\varphi_{s,a,t,\nu},\psi_{s,a,t,\nu})\}_{\nu\ge1}$ is a singular system of $\mathcal T_{t\mid(s,a)}$. Then for all $(s,a)\in \mathcal {S\times A}$ and $t = 1,\dots, T$, \[\sum_{\nu\ge1}\frac{\big\langle g,\psi_{s,a,t,\nu}\big\rangle_{\mathcal L^2(P_{R_t\mid s,a})}^2}{\sigma_{s,a,t,\nu}^2}<\infty,\] 
    where $\big\langle g,\psi_{s,a,t,\nu}\big\rangle_{\mathcal L^2(P_{R_t\mid s,a})} = \int_{\mathcal R} g(r) \cdot \psi_{s,a,t,\nu}(r)p_{R_t\mid s,a}(r) dr$.
\end{assumption}

\begin{lemma}[Picard's Theorem \cite{kress1989linear}]\label{lem:picard}
    Let $\mathcal H_1,\mathcal H_2$ be real Hilbert spaces and 
$K:\mathcal H_1\to\mathcal H_2$ a compact linear operator with adjoint $K^*:\mathcal H_2\to\mathcal H_1$.
Then, there exists a singular system $\{(\lambda_\nu,\phi_\nu,\psi_\nu)\}_{\nu=1}^\infty$ of $K$, 
with singular values $\lambda_\nu>0$ and orthonormal sequences $\{\phi_\nu\}\subset\mathcal H_1$,
$\{\psi_\nu\}\subset\mathcal H_2$ satisfying
\[
K\phi_\nu=\lambda_\nu\psi_\nu,\quad K^*\psi_\nu=\lambda_\nu\phi_\nu.
\]

Given $g\in\mathcal H_2$, the first–kind Fredholm equation $Kh=g$ has a solution 
$h\in\mathcal H_1$ if and only if
\begin{enumerate}
    \item[(a)] $g\in \ker(K^*)^\perp$;
    \item[(b)] $\sum_{\nu=1}^\infty \lambda_\nu^{-2}|\langle g,\psi_\nu\rangle_{\mathcal H_2}|^2<\infty$,
\end{enumerate}
where $\ker(K^*) = \{h: K^*h=0\}$ is the null space of $K^*$, and $\perp$ denotes the orthogonal complement to a set.
\end{lemma}

\begin{proposition}[Existence of bridges]
    Under \cref{ass:complete} (1), \cref{ass:integral,ass:sumfinite}, for all $(s,a)\in \mathcal {S\times A}$ and $t = 1,\dots, T$, there exists a solution $h$ to equation \[\mathcal T_{t\mid (s,a)}h = g,\] where $h := b_t(s,a;\cdot)\in \mathcal L^2\big(P_{S_{t+1}\mid s,a}\big)$ and $g(r)=r\in\mathcal L^2(P_{R_t\mid s,a})$. Equivalently, for any fixed $(s,a)$ and $t$, there exists a function $b_t(s,a,\cdot)$ satisfying \cref{eq:bridge}.
\end{proposition}
\begin{proof}
    By \cref{ass:integral}, for all $(s,a)\in \mathcal {S\times A}$ and $t = 1,\dots, T$, the operator $\mathcal T_{t\mid(s,a)}$ is Hilbert–Schmidt and thus compact. Suppose there exists function $f\in \ker(\mathcal{T}_{t\mid(s,a)}^*)$, then by definition, $\mathcal{T}_{t\mid(s,a)}^*f = 0$. By \cref{ass:complete} (1), we have $f(R_t)=0,\ a.s.$ Therefore, $\ker(\mathcal{T}_{t\mid(s,a)}^*) = \{0\}$, and hence $\ker(\mathcal{T}_{t\mid(s,a)}^*)^\perp = \mathcal L^2\big(P_{R_t\mid s,a}\big)$. Because reward $R_t$ is bounded, then $g \in \mathcal L^2\big(P_{R_t\mid s,a}\big)$. So the condition (a) in \cref{lem:picard} is satisfied. 
    
    Additionally, condition (b) is also satisfied by \cref{ass:sumfinite}. By \cref{lem:picard}, there exists a solution $h \in \mathcal{L}^2 (P_{S_{t+1} \mid s,a})$ to $\mathcal T_{t\mid (s,a)}h = g$ where $g(r) = r$, i.e., there exists function $b_t(s,a,\cdot)$ such that:
\[
\mathbb E\big[b_t(s,a,S_{t+1})\mid R_t=r,S_t=s,A_t=a\big]=r.
\]
\end{proof}

\subsection*{Part C: Uniqueness of bridge functions}
Next we study the uniqueness of the bridge functions.

\begin{proposition}[Uniqueness of bridges]
Under \cref{ass:complete} (2), any bridge function $b_t$  satisfying 
\[\mathbb E[b_t(S_t, A_t,S_{t+1})\mid R_t,S_t,A_t]= R_t,\quad a.s.\]
is unique. 
\end{proposition}

\begin{proof}
    Suppose there exist different bridge functions $b_{1,t}$ and $b_{2,t}$ satisfying the equation above for any $t$. Then, 
    \[\mathbb E[b_{1,t}(S_t, A_t,S_{t+1}) - b_{2,t}(S_t, A_t,S_{t+1})\mid R_t,S_t,A_t]= 0,\quad a.s.\]
    By \cref{ass:complete} (2), we have
    \[b_{1,t} - b_{2,t}=0,\quad a.s.,\;\; \forall t = 1,\dots, T, \]
    which contradicts $b_{1,t}\neq b_{2,t}$. Thus uniqueness holds.
\end{proof}

\section{Proof of \cref{thm:projbridgeerr}}
For a function class $\mathcal G$ and radius $\delta>0$, given sample $\{X_i\}$, the \textit{local empirical Rademacher complexity} is defined by
\[\widehat {\mathcal R}_{n}(\mathcal G,\delta)
=\mathbb E_{\varepsilon}\Big[\sup_{g\in\mathcal G:\|g\|_{2,n}\le \delta}\big|\frac1{n}\sum_{i=1}^{n}\varepsilon_i g(X_i)\big|\,\Big|\, \{X_i\}\Big]\] where $\|g\|_{2,n}^2 = \frac1n\sum_{i=1}^n g(X_i)^2$. The empirical critical radius $\hat \delta_n$ is the smallest solution to the inequality $\widehat {\mathcal R}_{n}(\mathcal G,\delta)\le \frac{\delta^2}{b}$. \citet{wainwright2019high} gives the relationship of critical radius and empirical critical radius: with probability at least $1-\zeta$, 
\[\delta_n\le \mathcal O(\hat\delta_n+\sqrt{\frac{\log(1/\zeta)}{n}}),\]
which enables us to study on empirical critical radius $\hat \delta_n$.

Define \[\Psi_{n}^t(b,g) = \frac1{n_t}\sum_{i\in \mathcal I_t^{\mathrm{obs}}}(b_t(S_{t,i}, A_{t,i}, S_{t+1,i}) - R^{\text{obs}}_{t,i}) g_t(R^{\text{obs}}_{t,i}, S_{t,i}, A_{t,i}),\] and population level version \[\Psi^{t}(b,g) = \mathbb E\Big[(b_t(S_t, A_t, S_{t+1}) - R_t) g_t(R_{t}, S_{t}, A_{t})\,\big|\, O_t=1\Big].\] Moreover, Let \[\Psi_{n}^{t,\lambda} (b,g) = \Psi_{n}^{t}(b,g) - \lambda \left( \|g_t\|_{\mathcal G^{(t)}}^2 + \frac{U}{\delta^2} \|g_t\|_{2,n_t}^2\right),\] \[\Psi^{t,\lambda} (b,g) = \Psi^{t}(b,g) - \lambda \left( \frac{2}{3} \|g_t\|_{\mathcal G^{(t)}}^2 + \frac{U}{2\delta^2} \|g_t\|_2^2\right).\] So the minimizer $\hat b_t$ can be written as
\[\hat b_t = \arg\min_{b_t\in \mathcal B^{(t)} }\sup_{g_t\in\mathcal G^{(t)}}\Psi_{n}^{t,\lambda}(b,g)+\lambda\mu\|b_t\|_{\mathcal B^{(t)}}^2.\]

By \cref{lem:141}, with probability at least $1-\zeta$, for any function $g_t\in \mathcal G^{(t)}_{3U}$, 
\[\bigl|\|g_t\|_{2,n_t}^2-\|g_t\|_{2}^2\bigr|\le \frac12\|g_t\|_{2}^2+(\delta_t^{\mathcal G})^2,\]
where $\delta_t^{\mathcal G} = \delta_{n_t}^{\mathcal G} + c_0\sqrt{\frac{\log(c_1/\zeta)}{n_t}}$, and $\delta_{n_t}^{\mathcal G}$ is the upper bound of the empirical critical radii of function class $\mathcal G^{(t)}_{3U}$. For any $\|g_t\|_{\mathcal G^{(t)}}^2\ge 3U$, consider rescaling $g_t$ by $\frac{\sqrt{3U}}{\|g_t\|_{\mathcal G^{(t)}}}g_t \in \mathcal G^{(t)}_{3U}$, and we have
\[\bigl|\|g_t\|_{2,n_t}^2-\|g_t\|_{2}^2\bigr|\le \frac12\|g_t\|_{2}^2+(\delta_t^{\mathcal G})^2\frac{\|g_t\|_{\mathcal G^{(t)}}^2}{3U}.\]
Combine the above inequalities, 
\begin{equation}\label{eq:gtbounded}
    \left| \|g_t\|_{2,n_t}^2 - \|g_t\|_2^2 \right|
\leq \frac{1}{2} \|g_t\|_2^2 + (\delta_t^{\mathcal G})^2\max\left\{1,\frac{\|g_t\|_{\mathcal G^{(t)}}^2}{3U}\right\}.
\end{equation}
Thus,
\begin{equation}\label{eq:ineqbridgeerr}
    \begin{aligned}
\|g_t\|_{\mathcal G^{(t)}}^2+\frac{U}{\delta^2}\|g_t\|_{2,n_t}^2
&\ge \|g_t\|_{\mathcal G^{(t)}}^2+\frac{U}{(\delta_t^{\mathcal G})^2}
\left[\frac{1}{2}\|g_t\|_2^2-(\delta_t^{\mathcal G})^2\max\left\{1,\frac{\|g_t\|_{\mathcal G^{(t)}}^2}{3U}\right\}\right]\\
&\ge \frac{2}{3}\|g_t\|_{\mathcal G^{(t)}}^2+\frac{U}{2(\delta_t^{\mathcal G})^2}\|g_t\|_2^2-U.
    \end{aligned}
\end{equation}

Next, we study the upper and lower bounds of the centered empirical sup-loss
\[\sup_{g_t\in\mathcal G^{(t)}} \Psi_{n}^{t}(\hat b_t,g)-\Psi_{n}^{t}(b_t^*,g)
-2\lambda\left(\|g_t\|_{\mathcal G^{(t)}}^2+\frac{U}{(\delta_t^{\mathcal G})^2}\|g_t\|_{2,n_t}^2\right).\]
For simplicity we omit $t$ and write it as 
\begin{equation}
\label{eq:BoundObj-b}
\sup_{g\in\mathcal G} \Psi_{n}(\hat b,g)-\Psi_{n}(b^*,g)
-2\lambda\left(\|g\|_{\mathcal G}^2+\frac{U}{\delta^2}\|g\|_{2,n}^2\right).
\end{equation}

\subsection{Upper bounding the centered empirical sup-loss}
We first decompose $\Psi_{n}^{\lambda}(b,g)$ by 
\begin{align*}
\Psi_{n}^{\lambda}(b,g)
&=\Psi_{n}(b,g)-\Psi_{n}(b^*,g)
+\Psi_{n}(b^*,g)-\lambda\left(\|g\|_{\mathcal G}^2+\frac{U}{\delta^2}\|g\|_{2,n}^2\right)\\
&\ge \Psi_{n}(b,g)-\Psi_{n}(b^*,g)
-2\lambda\left(\|g\|_{\mathcal G}^2+\frac{U}{\delta^2}\|g\|_{2,n}^2\right)
-\sup_{g\in\mathcal G}\Psi_{n}^{\lambda}(b^*,g),
\end{align*}
where the last inequality holds by symmetry of $\mathcal G$. Then we have 
\begin{equation}\label{eq:upperbridgeerr2}
    \begin{aligned}
        \sup_{g\in\mathcal G} \Psi_{n}(\hat b,g)-\Psi_{n}(b^*,g)
-2\lambda\left(\|g\|_{\mathcal G}^2+\frac{U}{\delta^2}\|g\|_{2,n}^2\right)
&\le \sup_{g\in\mathcal G}\Psi_{n}^{\lambda}(b^*,g)+\Psi_{n}^{\lambda}(\hat b,g)\\
&\le \sup_{g\in\mathcal G}\Psi_{n}^{\lambda}(b^*,g)
+\bigl[\sup_{g\in\mathcal G}\Psi_{n}^{\lambda}(b^*,g)+\lambda\mu\big(\|b^*\|_{\mathcal B}^2-\|\hat b\|_{\mathcal B}^2\big)\bigr]\\
&\le 2\sup_{g\in\mathcal G}\Psi_{n}^{\lambda}(b^*,g)
+\lambda\mu\big(\|b^*\|_{\mathcal B}^2-\|\hat b\|_{\mathcal B}^2\big).
    \end{aligned}
\end{equation}
By \cref{lem:foster14}, for all $g\in \mathcal G$, similarly, with probability at least $1-\zeta$ we have
\begin{equation}\label{eq:upperbridgeerr}
    \begin{aligned}
        \big|\Psi_n(b^*, g) - \Psi(b^*,g)\big|&\le 36 \delta\big[\|g\|_2+\delta\max\left\{1,\frac{\|g\|_{\mathcal G}}{\sqrt{3U}}\right\}\big]\\
        &\le 36 \delta \big[\|g\|_2 + \delta(1 + \frac{\|g\|_{\mathcal G}}{\sqrt{3U}})\big].
    \end{aligned}
\end{equation}
Combine \cref{eq:ineqbridgeerr} and \cref{eq:upperbridgeerr}, we have with probability at least $1-2\zeta$, for all $b\in \mathcal B$ and $g\in \mathcal G$, 
\begin{align*}
    \Psi_n^\lambda(b^*,g) &= \Psi_n(b^*,g) - \lambda \left( \|g\|_{\mathcal G}^2 + \frac{U}{\delta^2} \|g\|_{2,n}^2\right)\\
    &\le \Psi(b^*,g) +|\Psi_n(b^*,g)-\Psi(b^*,g)| - \lambda \left( \|g\|_{\mathcal G}^2 + \frac{U}{\delta^2} \|g\|_{2,n}^2\right)\\
    &\le \Psi(b^*,g) + 36 \delta\big[\|g\|_2 + \delta(1 + \frac{\|g\|_{\mathcal G}}{\sqrt{3U}})\big] - \lambda \left( \|g\|_{\mathcal G}^2 + \frac{U}{\delta^2} \|g\|_{2,n}^2\right)\\
    &\le \Psi(b^*,g) + 36 \delta\big[\|g\|_2 + \delta(1 + \frac{\|g\|_{\mathcal G}}{\sqrt{3U}})\big] - \lambda \left( \frac23\|g\|_{\mathcal G}^2 + \frac{U}{2\delta^2} \|g\|_{2,n}^2\right)+\lambda U\\
    &\le \Psi(b^*,g)-\lambda \left( \frac13\|g\|_{\mathcal G}^2+ \frac{U}{4\delta^2} \|g\|_{2,n}^2\right) + 36\delta^2 + \lambda U + 36 \delta\|g\|_2 + 36\delta^2 \frac{\|g\|_{\mathcal G}}{\sqrt{3U}} - \lambda \frac{U}{4\delta^2} \|g\|_{2,n}^2 - \frac\lambda3\|g\|_{\mathcal G}^2\\
    &= \Psi^{\lambda/2}(b^*,g) + 36\delta^2 + \lambda U + \bigl(36 \delta\|g\|_2- \lambda \frac{U}{4\delta^2} \|g\|_{2,n}^2\bigr) + \bigl(36\delta^2 \frac{\|g\|_{\mathcal G}}{\sqrt{3U}}  - \frac\lambda3\|g\|_{\mathcal G}^2\bigr).\\
\end{align*}

Using the fact that for any $a,b>0$ and any norm, $\sup_{g\in\mathcal G}(a\|g\| - b\|g\|^2)\le \frac{a^2}{4b}$, suppose $\lambda\ge \frac{C_1\delta^2}{U}$, and then
\begin{align*}
    \sup_{g\in\mathcal G}\bigl(36 \delta\|g\|_2- \lambda \frac{U}{4\delta^2} \|g\|_{2,n}^2\bigr)&\le \frac{36^2\delta^4}{\lambda U}\le \frac{36^2\delta^2}{C_1};\\
    \sup_{g\in\mathcal G}\bigl(36\delta^2 \frac{\|g\|_{\mathcal G}}{\sqrt{3U}}  - \frac\lambda3\|g\|_{\mathcal G}^2\bigr)&\le \frac{18^2\delta^4}{\lambda U}\le \frac{18^2\delta^2}{C_1}.
\end{align*}
Therefore, 
\[\Psi_n^\lambda(b^*,g)\le \Psi^{\lambda/2}(b^*,g) + 36\delta^2 + \lambda U + \frac{5\times 18^2\delta^2}{C_1}\]

Combine this with \cref{eq:upperbridgeerr2} we get
\begin{equation}\label{eq:upperbridgefinal}
    \begin{aligned}
\sup_{g\in\mathcal G} \Psi_{n}(\hat b,g)-\Psi_{n}(b^*,g)
-2\lambda\left(\|g\|_{\mathcal G}^2+\frac{U}{\delta^2}\|g\|_{2,n}^2\right)
&\le 2\sup_{g\in\mathcal G}\Psi_{n}^{\lambda}(b^*,g)
+\lambda\mu\big(\|b^*\|_{\mathcal B}^2-\|\hat b\|_{\mathcal B}^2\big)\\
&\le 2\sup_{g\in\mathcal G}\Psi^{\lambda/2}(b^*,g) + 2\lambda U + (72+\frac{10\times 18^2}{C_1})\delta^2+\lambda\mu\big(\|b^*\|_{\mathcal B}^2-\|\hat b\|_{\mathcal B}^2\big)\\
&=2\lambda U + (72+\frac{10\times 18^2}{C_1})\delta^2+\lambda\mu\big(\|b^*\|_{\mathcal B}^2-\|\hat b\|_{\mathcal B}^2\big).
    \end{aligned}
\end{equation}

\subsection{Lower bounding the centered empirical sup-loss}
By \cref{ass:richness}, we write $g_b = \arg\inf_{g\in \mathcal G_{L^2\|b-b^*\|_{\mathcal B}^2}}
\|g-\mathcal T(b-b^*)\|_2$ where \[\sup_{b\in\mathcal B}\inf_{g\in \mathcal G_{L^2\|b-b^*\|_{\mathcal B}^2}}
\|g-\mathcal T(b-b^*)\|_2\le \eta<\infty.\]
Also, let $g_{\hat b} = \arg\inf_{g\in \mathcal G_{L^2\|\hat b-b^*\|_{\mathcal B}^2}}
\|g-\mathcal T(\hat b-b^*)\|_2$.

When $\|g_{\hat b}\|_2\le \delta$, then 
\[\|\mathcal T(\hat b-b^*)\|_2\le \|g_{\hat b}\|+\|g-\mathcal T(b-b^*)\|_2\le \delta+\eta;\]
if $\|g_{\hat b}\|_2\ge \delta$, let $r = \frac{\delta}{2\|g_{\hat b}\|_2}\in [0,\frac12]$, and $rg_{\hat b}\in\mathcal G_{L^2\|\hat b-b^*\|_{\mathcal B}^2}$ since $\mathcal G$ is star-shaped. Hence,
    \begin{align*}
        \sup_{g\in\mathcal G} \Psi_{n}(\hat b,g)-\Psi_{n}(b^*,g)
-2\lambda\left(\|g\|_{\mathcal G}^2+\frac{U}{\delta^2}\|g\|_{2,n}^2\right)
&\ge \Psi_{n}(\hat b,rg_{\hat b})-\Psi_{n}(b^*,rg_{\hat b})
-2\lambda\left(\|rg_{\hat b}\|_{\mathcal G}^2+\frac{U}{\delta^2}\|rg_{\hat b}\|_{2,n}^2\right)\\
&=\underbrace{r\left(\Psi_{n}(\hat b,g_{\hat b})-\Psi_{n}(b^*,g_{\hat b})\right)}_{(i)} - 2\lambda \underbrace{r^2\left(\|g_{\hat b}\|_{\mathcal G}^2+\frac{U}{\delta^2}\|g_{\hat b}\|_{2,n}^2\right)}_{(ii)}.
    \end{align*}
For $(ii)$, 
\begin{align}\label{eq:lowerbridgeerr1}
    r^2\left(\|g_{\hat b}\|_{\mathcal G}^2+\frac{U}{\delta^2}\|g_{\hat b}\|_{2,n}^2\right)
    &\le \frac14\|g_{\hat b}\|_{\mathcal G}^2+\frac{r^2U}{\delta^2}\|g_{\hat b}\|_{2,n}^2.
\end{align}
By \cref{eq:gtbounded}, with probability at least $1-\zeta$,
\[\frac{r^2U}{\delta^2}\|g_{\hat b}\|_{2,n}^2\le \frac{r^2U}{\delta^2}\Bigl[\|g_{\hat b}\|_{2}^2+\big|\|g_{\hat b}\|_{2,n}^2 - \|g_{\hat b}\|_{2}^2\big|\Bigr]\le \frac{r^2U}{\delta^2}\bigl(\frac32\|g_{\hat b}\|_{2}^2+\delta^2(1+\frac{\|g_{\hat b}\|_{\mathcal G}^2}{3U})\bigr).\]
Substitute this into \cref{eq:lowerbridgeerr1} and we get
\begin{equation}
    \begin{aligned}
        r^2\left(\|g_{\hat b}\|_{\mathcal G}^2+\frac{U}{\delta^2}\|g_{\hat b}\|_{2,n}^2\right)
    &\le \frac14\|g_{\hat b}\|_{\mathcal G}^2+\frac{r^2U}{\delta^2}\bigl(\frac32\|g_{\hat b}\|_{2}^2+\delta^2(1+\frac{\|g_{\hat b}\|_{\mathcal G}^2}{3U})\bigr)\\
    &\le (\frac14\|g_{\hat b}\|_{\mathcal G}^2 + \frac1{12}\|g_{\hat b}\|_{\mathcal G}^2)+\frac{3Ur^2}{2\delta^2}\|g_{\hat b}\|_2^2+\frac14 U\\
    & = \frac13\|g_{\hat b}\|_{\mathcal G}^2+(\frac38+\frac14) U\\
    &\le \frac13 L^2\|\hat b - b^*\|_\mathcal B^2+\frac58 U.
    \end{aligned}
\end{equation}

For $(i)$, we consider function class
\[\mathcal J_{B,L^2B} = \Big\{((s,a,s'),(r,s,a))\mapsto \alpha(b(s,a,s')-b^*(s,a,s'))g^{L^2B}_{b}(r,s,a)\mid b-b^* \in \mathcal B_{B}, \alpha\in[0,1]\Big\},\]
where $g^{L^2B}_{b}(r,s,a) = \arg\inf_{g\in \mathcal G_{L^2B}}\|g-\mathcal T(b-b^*)\|_2$. Choose $\delta = \delta_{n}^{\mathcal J} + c_0\sqrt{\frac{\log(c_1/\zeta)}{n}}$, where $\delta_{n}^{\mathcal J}$ is the upper bound of the empirical critical radii of function class $\mathcal J_{B,L^2B}$. Choose loss function $\mathcal L = (b-b^*)f$, then by \cref{lem:foster14}, we have with probability at least $1-\zeta$, for all $b\in \mathcal B$ and $g\in\mathcal G$, 
\begin{align*}
    |(\Psi_n(b, g_b) - \Psi_n(b^*,g_b)) - (\Psi(b, g_b) - \Psi(b^*,g_b))|&\le 18\delta(\|(b^*-b)g_b\|_2+\delta)\\
    &\le 18\delta(\|g_b\|_2+\delta),
\end{align*}
since $b-b^*\in \mathcal B_B$, which is $1$-uniformly bounded.

When $\|b-b^*\|_{\mathcal B}^2> B$, we rescale the function by $\frac{\sqrt{B}}{\|b-b^*\|_{\mathcal B}}(b-b^*)$, and similarly, we obtain that with probability at least $1-\zeta$,
\begin{align*}
|(\Psi_n(b, g_b) - \Psi_n(b^*,g_b)) - (\Psi(b, g_b) - \Psi(b^*,g_b))|&\le 18\delta(\|g_b\|_2+\delta)\max \Bigl\{1,\frac{\|b-b^*\|_{\mathcal B}^2}{B}\Bigr\}.
\end{align*}

Therefore, with probability at least $1-\zeta$, for any $g\in \mathcal G$,
\begin{align*}
    r\left(\Psi_{n}(\hat b,g_{\hat b})-\Psi_{n}(b^*,g_{\hat b})\right)&\ge r\left(\Psi(\hat b,g_{\hat b})-\Psi(b^*,g_{\hat b})\right) - r|(\Psi_n(b, g_{\hat b}) - \Psi_n(b^*,g_{\hat b})) - (\Psi(b, g_{\hat b}) - \Psi(b^*,g_{\hat b}))|\\
    &\ge \underbrace{r\left(\Psi(\hat b,g_{\hat b})-\Psi(b^*,g_{\hat b})\right)}_{(A)} - \underbrace{18\delta r(\|g_{\hat b}\|_2+\delta)\max \Bigl\{1,\frac{\|\hat b-b^*\|_{\mathcal B}^2}{B}\Bigr\}}_{(B)}.
\end{align*}

For $(A)$, 
\begin{align*}
    r\left(\Psi(\hat b,g_{\hat b})-\Psi(b^*,g_{\hat b})\right)&= \frac{\delta}{2\|g_{\hat b}\|_2}\mathbb E\Big[(\hat b(S,A,S') - b^*(S,A,S'))g_{\hat b}(R,S,A)\,\big|\, O_t=1\Big]\\
    &= \frac{\delta}{2\|g_{\hat b}\|_2}\mathbb E\Big[g_{\hat b}(R,S,A) \mathbb E\left(\hat b(S,A,S') - b^*(S,A,S')\mid R,S,A, O_t=1\right)\,\big|\, O_t=1\Big]\\
    &= \frac{\delta}{2\|g_{\hat b}\|_2}\mathbb E\Big\{g_{\hat b}(R,S,A) \big[\mathcal T(\hat b - b^*)(R,S,A)\big]\Big\}\\
    &= \frac{\delta}{2\|g_{\hat b}\|_2}\mathbb E\big\{(g_{\hat b}(R,S,A))^2 - g_{\hat b}(R,S,A)[g_{\hat b}(R,S,A) -\mathcal T(\hat b - b^*)(R,S,A)]\big\}\\
    &= \frac{\delta}{2\|g_{\hat b}\|_2}\Big\{\|g_{\hat b}\|_2^2 -\bigl\{\mathbb E g_{\hat b}(R,S,A)[g_{\hat b}(R,S,A) -\mathcal T(\hat b - b^*)(R,S,A)]\bigr\}\Big\}\\
    &\ge \frac{\delta}{2\|g_{\hat b}\|_2}\Big\{\|g_{\hat b}\|_2^2 -\|g_{\hat b}\|_2\|g_{\hat b} -\mathcal T(\hat b - b^*))\|_2\Big\}\\
    &= \frac{\delta}{2}\Big\{\|g_{\hat b}\|_2 -\|g_{\hat b} -\mathcal T(\hat b - b^*))\|_2\Big\}\\
    &\ge \frac{\delta}{2}\Big\{\|\mathcal T(\hat b - b^*)\|_2 -2\|g_{\hat b} -\mathcal T(\hat b - b^*))\|_2\Big\}\\
    &\ge \frac{\delta}{2}\Big\{\|\mathcal T(\hat b - b^*)\|_2 -2\eta\Big\}.\\
\end{align*}

For $(B)$, 
\begin{align*}
    18\delta r(\|g_{\hat b}\|_2+\delta)\max \Bigl\{1,\frac{\|\hat b-b^*\|_{\mathcal B}^2}{B}\Bigr\}& = 18\delta r(\frac{\delta}{2r}+\delta)\max \Bigl\{1,\frac{\|\hat b-b^*\|_{\mathcal B}^2}{B}\Bigr\}\\
    & = (9\delta^2+18r\delta^2)\max \Bigl\{1,\frac{\|\hat b-b^*\|_{\mathcal B}^2}{B}\Bigr\}\\
    & \le 18\delta^2+\frac{18\delta^2\|\hat b-b^*\|_{\mathcal B}^2}{B}.
\end{align*}

So when $\|g_{\hat b}\|_2\ge \delta$, with probability at least $1-2\zeta$,
\begin{align*}
    (i) = r\left(\Psi_{n}(\hat b,g_{\hat b})-\Psi_{n}(b^*,g_{\hat b})\right)
    &\ge (A) - (B)\\
    &\ge \frac{\delta}{2}\Big\{\|\mathcal T(\hat b - b^*)\|_2 -2\eta\Big\} - 18\delta^2-\frac{18\delta^2\|\hat b-b^*\|_{\mathcal B}^2}{B}.
\end{align*}

Therefore the lower bound of $\sup_{g\in\mathcal G} \Psi_{n}(\hat b,g)-\Psi_{n}(b^*,g)
-2\lambda\left(\|g\|_{\mathcal G}^2+\frac{U}{\delta^2}\|g\|_{2,n}^2\right)$ is given by
\begin{align*}
    &\sup_{g\in\mathcal G} \Psi_{n}(\hat b,g)-\Psi_{n}(b^*,g)
-2\lambda\left(\|g\|_{\mathcal G}^2+\frac{U}{\delta^2}\|g\|_{2,n}^2\right)\\\ge &(i) - 2\lambda(ii)\\
\ge &\frac{\delta}{2}\Big\{\|\mathcal T(\hat b - b^*)\|_2 -2\eta\Big\} - 18\delta^2-\frac{18\delta^2\|\hat b-b^*\|_{\mathcal B}^2}{B} - 2\lambda\bigl( \frac13 L^2\|\hat b - b^*\|_\mathcal B^2+\frac58 U\bigr)\\
\ge &\frac{\delta}{2}\|\mathcal T(\hat b - b^*)\|_2 -\eta\delta - 18\delta^2-(\frac{18\delta^2}{B}+\frac{2\lambda L^2}{3})\|\hat b-b^*\|_{\mathcal B}^2 - \frac54\lambda U.
\end{align*}

\subsection{Combining the upper and lower bounds}
Combine the upper bound and lower bound, and then we have either $\|g_{\hat b}\|_2\le \delta$, or with probability at least $1-3\zeta$, for all $b\in \mathcal B$,
\[\frac{\delta}{2}\|\mathcal T(\hat b - b^*)\|_2 -\eta\delta - 18\delta^2-(\frac{18\delta^2}{B}+\frac{2\lambda L^2}{3})\|\hat b-b^*\|_{\mathcal B}^2 - \frac54\lambda U\le 2\lambda U + (72+\frac{10\times 18^2}{C_1})\delta^2+\lambda\mu\big(\|b^*\|_{\mathcal B}^2-\|\hat b\|_{\mathcal B}^2\big).\]
So
\begin{align*}
   \frac{\delta}{2}\|\mathcal T(\hat b - b^*)\|_2&\le \frac{13}4\lambda U+ \eta\delta+\lambda\mu\big(\|b^*\|_{\mathcal B}^2-\|\hat b\|_{\mathcal B}^2\big)+(90+\frac{10\times 18^2}{C_1})\delta^2+(\frac{18\delta^2}{B}+\frac{2\lambda L^2}{3})\|\hat b-b^*\|_{\mathcal B}^2\\
   &\le \frac{13}4\lambda U+ \eta\delta+\lambda\mu\big(\|b^*\|_{\mathcal B}^2-\|\hat b\|_{\mathcal B}^2\big)+(90+\frac{10\times 18^2}{C_1})\delta^2+2\lambda(\frac{18\delta^2}{\lambda B}+\frac{2 L^2}{3})\big(\|b^*\|_{\mathcal B}^2+\|\hat b\|_{\mathcal B}^2\big).
\end{align*}
If $\mu \ge \frac{36\delta^2}{\lambda B}+\frac{4 L^2}{3}$, then 
\begin{align*}
    \|\mathcal T(\hat b - b^*)\|_2&\le \frac{2}{\delta}\Big(\frac{13}4\lambda U+ \eta\delta+\lambda\mu\big(\|b^*\|_{\mathcal B}^2-\|\hat b\|_{\mathcal B}^2\big)+(90+\frac{10\times 18^2}{C_1})\delta^2+\lambda\mu\big(\|b^*\|_{\mathcal B}^2+\|\hat b\|_{\mathcal B}^2\big)\Big)\\
    &=\frac{2}{\delta}\Big(\frac{13}4\lambda U+ \eta\delta+2\lambda\mu\|b^*\|_{\mathcal B}^2+(90+\frac{10\times 18^2}{C_1})\delta^2\Big)\\
    &\le \frac{13}2\frac{\lambda U}{\delta}+ 2\eta+4\frac{\lambda\mu}{\delta}\|b^*\|_{\mathcal B}^2+(180+\frac{20\times 18^2}{C_1})\delta.
\end{align*}
Suppose $\lambda\le \frac{C_2\delta^2}{U}$, then with probability at least $1-4\zeta$, 
\begin{align*}
    \|\mathcal T(\hat b - b^*)\|_2
    &\le \frac{13}2\frac{\lambda U}{\delta}+ 2\eta+4\frac{\lambda\mu}{\delta}\|b^*\|_{\mathcal B}^2+(180+\frac{20\times 18^2}{C_1})\delta\\
    &\le \frac{13}2C_2\delta + 2\eta + 4C_2\mu\delta\|b^*\|_{\mathcal B}^2+(180+\frac{20\times 18^2}{C_1})\delta\\
    &\le 4C_2\mu\delta\|b^*\|_{\mathcal B}^2+\bigl(\frac{13}2C_2+180+\frac{20\times 18^2}{C_1}\bigr)\delta+2\eta\\
    &\lesssim \delta\max\Big\{1,\|b^*\|_{\mathcal B}^2\Big\}.
\end{align*}

\section{Proof of \cref{cor:bridgeerrrkhs}}
Recall that the product class is denoted by
\[\mathcal J^{(t)}_{B,L^2B}:=
\Big\{((s,a,s'),(r,s,a))\mapsto \alpha(b_t(s,a,s')-b_t^*(s,a,
s'))g^{L^2B}_{b,t}(r,s,a)\mid b_t-b_t^* \in \mathcal B^{(t)}_{B}, \alpha\in[0,1]\Big\}.\]
Define the tensor product of two RKHSs $\mathcal B^{(t)}$ and $\mathcal G^{(t)}$ as $\mathcal H_\otimes^{(t)}$ endowed with kernel $K_{\otimes,t}\big((x,y),(x',y')\big) := K_{B,t}(x,x')\,K_{G,t}(y,y')$. Then, one can verify that $\mathcal J^{(t)}_{B,L^2B}$ satisfies
\[\mathcal J_{B,L^2B}^{(t)}
\subseteq
\Big\{f\in\mathcal H_\otimes^{(t)}:\ \|f\|_{\mathcal H_\otimes^{(t)}}\le \sqrt{B_t}\sqrt{L^2B_t}\Big\}
=: \mathcal H_{\otimes,LB}^{(t)}.\]
By \cref{lem:rkhsradem}, we have that
\begin{equation}\label{eq:rademJLB}
    \mathcal R_{n_t}(\mathcal J_{B,L^2B}^{(t)},\delta)\le LB\sqrt{\frac2{n_t}}\sqrt{\sum_{i=1}^\infty\sum_{j=1}^\infty \min \big\{\mu_{t,i}^\mathcal B\mu_{t,j}^\mathcal G,\delta^2\big\}}.
\end{equation}
Under the polynomial eigen decay assumptions
$\mu^{\mathcal B}_{t,i}\lesssim i^{-2\alpha_B}$ and $\mu^{\mathcal G}_{t,j}\lesssim j^{-2\alpha_G}$ with
$\alpha_{\min}:=\min\{\alpha_B,\alpha_G\}$, by \cref{lem:kriegpoly}, the tensor-product spectrum admits the bound
\[
\sum_{i=1}^\infty\sum_{j=1}^\infty \min\{\mu_{t,i}^\mathcal B\mu_{t,j}^\mathcal G,\delta^2\}
\lesssim
\delta^{2-\frac{1}{\alpha_{\min}}}\log\frac{1}{\delta}.
\]
plugging this into \cref{eq:rademJLB} yields 
\[
\delta_{n_t}^{\mathcal J}\ \lesssim\ LB\; n_t^{-\frac{\alpha_{\min}}{2\alpha_{\min}+1}}\log n_t.
\]
Similarly, we have
\[
\delta_{n_t}^{\mathcal G}\ \lesssim\ \sqrt{U_t}\; n_t^{-\frac{\alpha_G}{2\alpha_G+1}}\log n_t.
\]
Therefore, $\delta_{n_t}=\max\{\delta_{n_t}^{\mathcal G},\delta_{n_t}^{\mathcal J}\}$ satisfies
\[\delta_{n_t}\lesssim \max\{\sqrt{U_t},LB\}\,n_t^{-\frac{\alpha_{\min}}{2\alpha_{\min}+1}}\log n_t,\]
and the claim follows by \cref{thm:projbridgeerr}.

The min-max estimation problem of $b_t$ has a closed-form solution, which is discussed in \citet{dikkala2020minimax} Appendix E.3.

\section{Proof of \cref{thm:policyerr}}\label{app:proof-policyerr}
\subsection{Error Decomposition}
The estimation error bound can be decomposed as
\begin{align*}
    |\mathbb E[V_1^\pi]-\widehat V(\pi)|&\le \underbrace{|\mathbb E[V_1^\pi] - \mathbb E_{n}[V_1^\pi]|}_{(I)}+\underbrace{|\mathbb E[V_1^\pi]- \mathbb E[\widehat V_1^\pi]|}_{(II)} + \underbrace{|\mathbb E(V_1^\pi - \widehat V_1^\pi) - \mathbb E_{n}(V_1^\pi - \widehat V_1^\pi)|}_{(III)}, 
\end{align*}
where $\mathbb E_{n}[V_1^\pi] := \frac1n\sum_{i=1}^n V_1^\pi(S_{1,i},0)$.

\subsection{Bound of $(I)$}
For $(I)$, since rewards are bounded in $[-1,1]$, by Hoeffding inequality, with probability at least $1-\zeta$,
\[|\mathbb E[V_1^\pi] - \mathbb E_{n}[V_1^\pi]|\lesssim \|V_1^\pi\|_{\infty}\sqrt{\frac{\log(c_1/\zeta)}{n}}\le T\sqrt{\frac{\log(c_1/\zeta)}{n}},\]
where constant $c_1>0$.
\subsection{Bound of $(II)$}
For $(II)$,
\begin{align*}
    |\mathbb E[V_1^\pi]- \mathbb E[\widehat V_1^\pi]|&\le \|V_1^\pi- \widehat V_1^\pi\|_{2}\\
    & = \|\sum_a\pi_1(a\mid S_1, O_{0})(Q_1^\pi(S_1,a) - \widehat Q_1(S_1,a))\|_{2}\\
    & = \Big(\mathbb E\big[\big(\sum_a\pi_1(a\mid S_1, O_{0}=0)(Q_1^\pi(S_1,a) - \widehat Q_1(S_1,a))\big)^2\big]\Big)^{1/2}\\
    & \le \Big(\mathbb E\big[\sum_a\pi_1(a\mid S_1, O_{0}=0)(Q_1^\pi(S_1,a)- \widehat Q_1(S_1,a))^2\big]\Big)^{1/2}\\
    & := \|Q_1^\pi -\widehat  Q_1\|_{2,\pi},
\end{align*}
where $Q_t^\pi(s,a) = \mathbb E(R_t+ V_{t+1}^\pi\mid s,a)$. Since there is no shift on marginal distributions of states $\tilde d_t^\pi$ and $\tilde d_t^b$ when $t=1$, by \cref{ass:concentrability}, 
\begin{align*}
   \|Q_1^\pi -\widehat  Q_1\|_{2,\pi}^2 &= \mathbb E_{(S_1,O_0)\sim \tilde d_1^\pi}\big[\sum_a\pi_1(a\mid S_1, O_{0}=0)(Q_1^\pi (S_1,a)- \widehat Q_1(S_1,a))^2\big]\\
   &=\mathbb E_{(S_1,O_0)\sim \tilde d_1^b}\big[\sum_a\pi_1(a\mid S_1, O_{0}=0)(Q_1^\pi (S_1,a)- \widehat Q_1(S_1,a))^2\big]\\
   &\le \kappa_1\mathbb E_{(S_1,O_0)\sim \tilde d_1^b}\big[\sum_a\pi_1^b(a\mid S_1)(Q_1^\pi (S_1,a)- \widehat Q_1(S_1,a))^2\big]\\
   &:= \kappa_1\|Q_1^\pi -\widehat  Q_1\|^2_{2,\pi^b},
\end{align*}
and for simplicity we write $\|Q_1^\pi -\widehat Q_1\|_{2}^2$ instead of $\|Q_1^\pi -\widehat Q_1\|_{2,\pi^b}^2$. $\widehat Q_t$ is estimated from penalized nonparametric least square problem \cref{eq:fitq}. 

$\|\widehat Q_t - Q_t^\pi\|_2$ involves the estimation error of the fitted $Q$-functions produced by FQE. Unlike standard supervised regression, the regression targets in FQE are pseudo-labels that depend on nuisance estimates and on future-stage fitted values. Concretely, at stage $t<T$,
the target takes the form
\[
y_{t,i}=\widehat{\widetilde R}_{t,i}+\widehat V_{t+1}^{\pi}(S_{t+1,i},O_{t,i}),
\]
where $\widehat{\widetilde R}_{t,i}$ depends on the estimated bridge $\hat b_t$ and
$\widehat V_{t+1}^{\pi}$ depends on $\widehat Q_{t+1}$.
Therefore, the regression noise and the regression function are statistically coupled through
the common data, and a direct analysis of $\|\widehat Q_t-Q_t^\pi\|_2$ typically leads to
non-negligible cross terms that are difficult to control without additional device such as
sample splitting or cross-fitting. 

To decouple the effect of nuisance estimation from the intrinsic regression error, we introduce an oracle comparator $\widehat Q_t^*$, defined as the solution of the same
penalized regression problem as $\widehat Q_t$ but trained on an oracle pseudo-label $y_t^*$, in which the nuisance components are replaced by their population counterparts.
\begin{equation}\label{eq:fitqstar}
    \widehat Q_t^* = \arg\min_{f\in\mathcal Q^{(t)}}
\frac1{n}\sum_{i=1}^n
\big(f(S_{t,i},A_{t,i})-y_{t,i}^*\big)^2
+\lambda_{Q,t}\|f\|_{\mathcal Q^{(t)}}^2, 
\end{equation}
where 
\[y_{t,i}^{*}
=
\begin{cases}
\widetilde R_{t,i}, & t=T,\\
\widetilde R_{t,i}+V_{t+1}^{\pi}(S_{t+1,i},O_{t,i}), & t<T,
\end{cases}\]
and $V_{t+1}^{\pi}(S_{t+1,i},O_{t,i}) = \sum_a \pi_{t+1}(a\mid S_{t+1}, O_t)Q_{t+1}^\pi(S_{t+1},a)$.

Then we have 
\begin{align*}
    \|\widehat Q_t - Q_t^\pi\|_2&= \|(\widehat Q_t -\widehat Q_t^*)+(\widehat Q_t^*- Q_t^\pi)\|_2\\
    &\le \underbrace{\|\widehat Q_t -\widehat Q_t^*\|_2}_{(a)}+\underbrace{\|\widehat Q_t^*-Q_t^\pi\|_2}_{(b)}.
\end{align*}

For $(a)$, since $\widehat Q_t$, $\widehat Q_t^*$ are estimated from \cref{eq:fitq} and $\cref{eq:fitqstar}$, it can be verified that
\[\|\widehat Q_t - \widehat Q_t^*\|_{2,n}\le \|y_t - y_t^*\|_{2,n}\le \sqrt 2\|\widehat {\widetilde R}_t - \widetilde R_t\|_{2,n}+ \sqrt 2\|\widehat V_{t+1}^{\pi}-V_{t+1}^{\pi}\|_{2,n}.\]

Since 
\[\|\widehat {\widetilde R}_t - \widetilde R_t\|_{2,n} = \|(1-O_t)(\hat b_t-b_t^*)\|_{2,n}\le \|\hat b_t-b_t^*\|_{2,n},\]
and
\begin{align*}
    \|\widehat V_{t+1}^{\pi}-V_{t+1}^{\pi}\|_{2,n} 
    :=& \Big\|\sum_a\pi_{t+1}(a\mid S_{t+1}, O_{t})\bigl(\widehat Q_{t+1}(S_{t+1},a) - Q_{t+1}^\pi(S_{t+1},a)\bigr)\Big\|_{2,n}\\
    =& \Big(\frac1n \sum_{i=1}^n\big[\sum_a\pi_{t+1}(a\mid S_{t+1,i}, O_{t,i})\bigl(\widehat Q_{t+1}(S_{t+1,i},a) - Q_{t+1}^\pi (S_{t+1,i},a)\bigr)\big]^2\Big)^{1/2}\\
    \le& \Big(\frac1n \sum_{i=1}^n\sum_a\pi_{t+1}(a\mid S_{t+1,i}, O_{t,i})\bigl(\widehat Q_{t+1}(S_{t+1,i},a) - Q_{t+1}^\pi(S_{t+1,i},a)\bigr)^2\Big)^{1/2}\\
     :=& \|\widehat Q_{t+1} -  Q_{t+1}^\pi\|_{2,n,\pi},
\end{align*}
then
\begin{equation}\label{eq:empqerr}
    \|\widehat Q_t - \widehat Q_t^*\|_{2,n}\le\sqrt 2 \|\hat b_t-b_t^*\|_{2,n}+\sqrt 2\|\widehat Q_{t+1} - Q_{t+1}^\pi\|_{2,n,\pi}.
\end{equation}

Since $\mathcal Q^{(t)}$ is $(T-t+1)$-uniformly bounded, we consider scaling $\mathcal Q^{(t)}$ by $(T-t+1)$ for convenience.

According to \citet{fischer2020sobolev}, under mild conditions, the RKHS norm of kernel ridge regression estimators is bounded with high probability. and construct RKHS ball $\mathcal Q^{(t)}_{R_Q} = \{Q\in \mathcal Q^{(t)}: \|Q\|_{\mathcal Q^{(t)}}\le R_{Q}\}$ for all $t=1,\dots, T$. Denote difference class $\Delta \mathcal Q^{(t)} := \{\Delta_Q\mid \Delta_Q =Q_1 - Q_2, \;\;Q_1,Q_2\in \mathcal Q^{(t)}_{R_Q}\}$. Let $\bar \delta_{\Delta_Q^{(t)},n}$ be the upper bound of the empirical critical radii of scaled function class $\Delta \mathcal Q^{(t)}$.

For RKHS $\mathcal B^{(t)}$, Proposition 9 in \citet{dikkala2020minimax} gives the closed form of the inner maximization, which implies that $\|\hat b_t\|_{\mathcal B^{(t)}}$ can be bounded by a constant $R_B$. Thus, consider RKHS ball $\mathcal B^{(t)}_{R_B} = \{b_t\in \mathcal B^{(t)}, \|b_t\|_{\mathcal B^{(t)}}\le R_B\}$. Define difference class $\Delta \mathcal B^{(t)} = \{\Delta_b=b_1 - b_2,\;\;b_1,b_2\in \mathcal B^{(t)}_{R_B}\}$ and let $\bar \delta_{\Delta_b^{(t)},n}$ be the upper bound of the empirical critical radii of $\Delta \mathcal B^{(t)}$. Then, we define $\bar \delta_{\Delta_t}=\max \{\bar \delta_{\Delta_Q^{(t)},n}, \bar \delta_{\Delta_Q^{(t+1)},n},\bar \delta_{\Delta_b^{(t)},n}\}$
where $\delta_{\Delta_t} = \bar \delta_{\Delta_t} + c_0\sqrt{\frac{\log(c_1/\zeta)}{n}}$ for some $c_0, c_1>0$. Let $
\|\widehat Q_{t+1} - Q_{t+1}^\pi\|_{2,b,\pi}^2
:=
\mathbb E_{(S_{t+1},O_{t})\sim \tilde d_{t+1}^b}\Big[
\sum_{a\in\mathcal A}\pi_{t+1}(a\mid S_{t+1},O_{t})
\big(\widehat Q_{t+1}(S_{t+1},a) - Q_{t+1}^\pi(S_{t+1},a)\big)^2
\Big]$. 

Applying \cref{lem:141} on both sides of \cref{eq:empqerr}, we have with probability at least $1-\zeta$, 
\begin{equation}
\begin{aligned}
 \|\widehat Q_t - \widehat Q_t^*\|_{2}
\lesssim& \|\hat b_t-b_t^*\|_{2}
+\|\widehat Q_{t+1} - Q_{t+1}^\pi\|_{2,b,\pi}
+(T-t+1)\delta_{\Delta_t}\\
(\cref{ass:concentrability}~(2))\lesssim &  \|\hat b_t-b_t^*\|_{2}
+
\sqrt{\kappa_{t+1}}\|\widehat Q_{t+1} - Q_{t+1}^\pi\|_{2}
+(T-t+1)\delta_{\Delta_t}.
\end{aligned}
\end{equation}

$(b)$ corresponds to a standard penalized least square estimation error. Since $\|Q_t^*\|_{\mathcal Q^{(t)}}$ is bounded, by \cref{lem:plsradembound}, with probability at least $1-\zeta$, $(b)$ is bounded by
\[\|\widehat Q_t^*-Q_t^\pi\|_2\lesssim (\delta_{\Delta_Q^{(t)}}+\sqrt{\lambda_{Q,t}})(T-t+1),\]
where $\delta_{\Delta_Q^{(t)}} = \delta_{\Delta_Q^{(t)},n}+ c_0\sqrt{\frac{\log(c_1T/\zeta)}{n}}$ for some $c_0,c_1>0$.

Therefore, with probability at least $1-\zeta/T$, 

\begin{align*}
    \|\widehat Q_t - Q_t^\pi\|_2
    &\le \|\widehat Q_t -\widehat Q_t^*\|_2+\|\widehat Q_t^*-Q_t^\pi\|_2\\
    &\lesssim \|\hat b_t-b_t^*\|_{2}
+
\sqrt{\kappa_{t+1}}\|\widehat Q_{t+1} - Q_{t+1}^\pi\|_{2}
+(T-t+1)\delta_{\Delta_t}+(\delta_{\Delta_Q^{(t)}}+\sqrt{\lambda_{Q,t}})(T-t+1).
\end{align*}
Applying backward induction from $t=T$ down to $t=1$ yields the bound for $(II)$ with probability at least $1-\zeta$:
\begin{align*}
    |\mathbb E[V_1^\pi]- \mathbb E[\widehat V_1^\pi]|\le& \|V_1^\pi- \widehat V_1^\pi\|_{2}\\ \le& \sqrt{\kappa_1}
    \|\widehat Q_1 - Q_1^\pi\|_2\\ \lesssim &\sum_{t=1}^T\Big(\prod_{j=1}^T\sqrt{\kappa_j}\Big)\Big[\|\hat b_t - b_t^*\|_{2}
+(T-t+1)\delta_{\Delta_t}+(\delta_{\Delta_Q^{(t)}}+\sqrt{\lambda_{Q,t}})(T-t+1)\Big]\\
\lesssim &K\sum_{t=1}^T\Big[\tau_t\delta_t(1+\|b^*_t\|_{\mathcal B^{(t)}}^2)
+(\delta_{\Delta_t}+\delta_{\Delta_Q^{(t)}}+\sqrt{\lambda_{Q,t}})(T-t+1)\Big].
\end{align*}

\subsection{Bound of $(III)$}
For $(III)$, we first define function class $\Delta\mathcal V^{(t)}
= \big\{ \Delta = V_1 - V_2\mid V_1, V_2 \in \mathcal V^{(t)}_{R_V}\big\}$, where $\mathcal V^{(t)}_{R_V}$ is a $(T-t+1)$-uniformly bounded function class of value functions at time $t$, induced from $\mathcal Q^{(t)}$ under operator $\Pi_t$: $\mathcal V^{(t)}_{R_V}=\{\Pi_tQ: Q\in \mathcal Q^{(t)}_{R_Q}\}$. Here linear operator $\Pi_t$ is defined as $(\Pi_t Q)(s,o-) = \sum\pi_t(a\mid s,o_-)Q(s,a)$. We choose the cost function as $\mathcal L(f(X),Y) = f(X)$ and apply \cref{lem:bound3}. We here also scale $\mathcal V^{(t)}_{R_V}$ by $(T-t+1)$.

Then, with probability at least $1-\zeta$,
\[|\mathbb E(V_1^\pi - \widehat V_1^\pi) - \mathbb E_{n}(V_1^\pi - \widehat V_1^\pi)|\lesssim \delta_{\Delta_V^{(1)}}(\|V_1^\pi - \widehat V_1^\pi\|_{2}+T\delta_{\Delta_V^{(1)}}),\]
where $\delta_{\Delta_V^{(1)}} = \bar \delta_{\Delta_V^{(1)},n}+c_0\sqrt{\frac{\log(c_1/\zeta)}{n}}$, and $\bar \delta_{\Delta_V^{(1)},n}$ is the upper bound of the empirical critical radii of scaled function class $\Delta\mathcal V^{(1)}$. Moreover, $\delta_{\Delta_V^{(1)}}\le \delta_{\Delta_Q^{(1)}}$. 

\subsection{Policy value error bound}
Combine the above inequalities, we obtain the policy value estimation error bound with probability at least $1-\zeta$:
\begin{align*}\label{eq:policyerr}
    \big|\widehat V(\pi) - V(\pi)\big|\le& (I)+(II)+(III)\\
    \lesssim& T\sqrt{\frac{\log(c_1T/\zeta)}{n}}+(\delta_{\Delta_V^{(1)}}+1)\|V_1^\pi- \widehat V_1^\pi\|_{2}+T(\delta_{\Delta_V^{(1)}})^2\\
    \lesssim& T\sqrt{\frac{\log(c_1T/\zeta)}{n}} + T(\delta_{\Delta_V^{(1)}})^2 + K(\delta_{\Delta_V^{(1)}}+1)\Big[\tau_t\delta_t(1+\|b^*_t\|_{\mathcal B^{(t)}}^2) + \\&(\delta_{\Delta_t} + \delta_{\Delta_Q^{(t)}} + \sqrt{\lambda_{Q,t}})(T-t+1)\Big]\\
    \lesssim& T\sqrt{\frac{\log(c_1T/\zeta)}{n}} + T(\delta_{\Delta_V^{(1)}})^2+K(\delta_{\Delta_V^{(1)}}+1)\Big[\tau_t\delta_t(1+\|b_t^*\|_{\mathcal B^{(t)}}^2) +\\& (T-t+1)\delta_{t,*}\Big]\\
    \lesssim& T\sqrt{\frac{\log(c_1T/\zeta)}{n}}+K\tau_{\max}T\sum_{t=1}^T\delta_{t,*},
\end{align*}   

where $\delta_{t,*}$ as the maximum of the critical radii of difference classes $\Delta \mathcal Q^{(t)}$, $\Delta \mathcal Q^{(t+1)}$, $\Delta\mathcal B^{(t)}$, and $\mathcal G_U^{(t)}$ for $t=1,\dots, T$, namely $\delta_{\Delta^{(t)}_Q}$, $\delta_{\Delta^{(t+1)}_Q}$, $\delta_{\Delta^{(t)}_B}$ and $\delta_{G^{(t)}}$.

With polynomial decay
\[\mu^{\mathcal Q}_{t,j}\lesssim j^{-2\alpha_Q},\quad
\mu^{\mathcal B}_{t,j}\lesssim j^{-2\alpha_B},\quad
\mu^{\mathcal G}_{t,j}\lesssim j^{-2\alpha_G},
\quad \alpha_Q, \alpha_B, \alpha_G>1/2,\]
the corresponding critical radii satisfy
\begin{align*}
    \delta_{\Delta^{(t)}_Q},\delta_{\Delta^{(t+1)}_Q}&\lesssim R_{Q}^{\frac{1}{2\alpha_{Q}+1}}n^{-\frac{\alpha_{Q}}{2\alpha_{Q}+1}}\log n,\\
    \delta_{\Delta^{(t)}_B}&\lesssim R_{B}^{\frac{1}{2\alpha_{B}+1}}n^{-\frac{\alpha_{B}}{2\alpha_{B}+1}}\log n,\\
    \delta_{G^{(t)}}&\lesssim U_t^{\frac{1}{2\alpha_{G}+1}}n^{-\frac{\alpha_{G}}{2\alpha_{G}+1}}\log n.
\end{align*}
Thus, the critical radius $\delta_{t, *}$ satisfies
\[\delta_{t,*}\lesssim \max\{\sqrt{R_{Q}},\sqrt{R_{B}},\sqrt{U_t} \}n^{-\frac{\alpha_{\min}}{2\alpha_{\min}+1}}\log n,\]
where $\alpha_{\min} = \min\{\alpha_Q,\alpha_B,\alpha_G\}$. Therefore, with probability at least $1-\zeta$, the policy value is bound by
\[\big|\widehat V(\pi) - V(\pi)\big|\lesssim K\tau_{\max}T^{2}\sqrt{\log(c_1 T/\zeta)}n^{-\frac{\alpha_{\min}}{2\alpha_{\min}+1}}\log n.\]

\section{Auxiliary lemmas}
\begin{lemma}[\citet{wainwright2019high}, Theorem 14.20]\label{lem:bound3}
    Suppose function class $\mathcal F$ is symmetric, 1-uniformly bounded, and star-shaped around $f^*$. Let $\delta^2_n\ge \frac{c}{n}$ be any solution to the inequality $\mathcal R_n(\mathcal F^*, \delta)\le \delta^2$, where $\mathcal F^* = \{f-f^*\mid f\in\mathcal F\}$. Suppose the cost function $\mathcal L (f(X), Y)$ is $L$-Lipschitz in its first argument $f(X)$. Then for all $f\in \mathcal F$, with probability at least $1-c_1e^{-c_2n\delta^2_n}$, we have
\[|\mathbb E_n\big(\mathcal L (f(x), y) - \mathcal L (f^*(x), y)\big) - \mathbb E\big(\mathcal L (f(x), y) - \mathcal L (f^*(x), y)\big)|\le 10L\delta_n\big(\|f-f^*\|_2+\delta_n\big).\]
\end{lemma}

\begin{lemma}[\citet{wainwright2019high}, Theorem 14.1]\label{lem:141}
Given a star-shaped and $b$-uniformly bounded function class $\mathcal F$, set $\delta_n>0$ be any solution to $\mathcal R(\mathcal F,\delta)\le \frac{\delta^2}{b}$. Then for any $t\ge \delta_n$, with probability at least $1-c_1\exp(-c_2\frac{nt^2}{b^2})$, we have
\[\big|\|f\|_{2,n}^2 - \|f\|_2^2\big|\le \frac12\|f\|_2^2+\frac12t^2\]
for all $f\in \mathcal F$. 
\end{lemma}

\begin{lemma}[\citet{foster2023orthogonal}, Lemma 14]\label{lem:foster14}
    Consider a $1$-uniformly bounded and star-shaped function class $\mathcal F$, and pick any $f^*\in \mathcal F$. Let $\delta^2_n\ge c_1\frac{\log(\log n)}{n}$ be any solution to the inequalities $\mathcal R_n(\mathcal F_t^*, \delta)\le \delta^2$ for all $t\in \{1,\dots, d\}$, where $\mathcal F_t^* = \{f_t-f_t^*\mid f_t\in\mathcal F|_t\}$. Assume $\mathcal L_f$ is $L$-Lipschitz in its first argument $f$ with respect to its $\ell_2$ norm. Then for all $f\in \mathcal F$, for some universal constants $c_2,c_3>0$, with probability at least $1-c_2e^{-c_3n\delta^2_n}$, we have
\[|\mathbb E_n\big(\mathcal L_f - \mathcal L_{f^*}) - \mathbb E\big(\mathcal L_f - \mathcal L_{f^*}\big)|\le 18Ld\delta_n\big(\|f-f^*\|_2+\delta_n\big).\]
The outcome $\hat f$ of constrained ERM satisfies that with the same probability,
\[\mathbb E_n\big(\mathcal L_{\hat f} - \mathcal L_{f^*})\le 18Ld\delta_n\big(\|\hat f-f^*\|_2+\delta_n\big).\]
\end{lemma}

\begin{lemma}[\citet{wainwright2019high}, Example 3.5]
Let $\varepsilon=(\varepsilon_1,\dots,\varepsilon_n)$ be i.i.d.\ Rademacher variables taking values in $\{-1,+1\}$ with equal probability.
Let $\mathcal A\subset\mathbb R^n$ be any (possibly infinite) bounded set, and define
\[
Z(\mathcal A) := \sup_{a\in\mathcal A}\ \langle a,\varepsilon\rangle
= \sup_{a\in\mathcal A}\ \sum_{k=1}^n a_k \varepsilon_k.
\]
Let $W(\mathcal A) := \sup_{a\in\mathcal A}\|a\|_2$. Then for all $t>0$,
\[
\mathbb P\Big(Z(\mathcal A)\ge \mathbb E[Z(\mathcal A)]+t\Big)
\le
\exp\!\Big(-\frac{t^2}{16\,W(\mathcal A)^2}\Big).
\]
Moreover, since $-Z(\mathcal A)=\inf_{a\in\mathcal A}\langle a,\varepsilon\rangle$ and the same argument applies,
\[
\mathbb P\Big(|Z(\mathcal A)-\mathbb E[Z(\mathcal A)]|\ge t\Big)
\le
2\exp\!\Big(-\frac{t^2}{16\,W(\mathcal A)^2}\Big).
\]
\end{lemma}

\begin{lemma}[\citet{wainwright2019high}, Corollary 14.5]\label{lem:rkhsradem}
    Let $\mathcal H$ be an RKHS with reproducing kernel $K$ and let $
\mathcal F := \{ f\in \mathcal H : \|f\|_{\mathcal H}\le 1\}$ be the unit ball. Let $\{\mu_j\}_{j=1}^\infty$ denote the non-increasing eigenvalues.
Then the local Rademacher complexity satisfies, for any $\delta>0$,
\[
\mathcal R_n(\mathcal F, \delta)
\le
\sqrt{\frac{2}{n}}
\left(\sum_{j=1}^{\infty}\min\{\mu_j,\delta^2\}\right)^{1/2}.
\]
Moreover, let $\{\hat \mu_j\}_{j=1}^n$ denote the eigenvalues of the renormalized
kernel matrix $\mathbf K\in\mathbb R^{n\times n}$ with entries
$\mathbf K_{ij}=K(x_i,x_j)/n$. Then the local empirical Rademacher complexity satisfies,
for any $\delta>0$,
\[
\widehat{\mathcal R}_n(\mathcal F, \delta)
\le
\sqrt{\frac{2}{n}}
\left(\sum_{j=1}^{n}\min\{\widehat \mu_j,\delta^2\}\right)^{1/2}.
\]
\end{lemma}

\begin{lemma}[\citet{krieg2018tensor}, Theorem 1(i)]\label{lem:kriegpoly}
Let $\sigma:\mathbb N\to\mathbb R_+$ be a non-increasing sequence with $\sigma(n)\to 0$.
For $d\in\mathbb N$, define its $d$-th tensor power
\[
\sigma_d(n_1,\dots,n_d)=\prod_{k=1}^d \sigma(n_k),\quad (n_1,\dots,n_d)\in\mathbb N^d,
\]
and let $\tau:\mathbb N\to\mathbb R_+$ be the non-increasing rearrangement of
$\{\sigma_d(n_1,\dots,n_d)\}_{(n_1,\dots,n_d)\in\mathbb N^d}$.
If for some $s>0$ one has $\sigma(n)\lesssim n^{-s}$, then
\[
\tau(n)\lesssim n^{-s}(\log n)^{s(d-1)}.
\]
\end{lemma}

\begin{lemma}[Rademacher analogue of \citet{wainwright2019high}, Theorem 13.17]
    \label{lem:plsradembound}
Let $(x_i,y_i)_{i=1}^n$ be i.i.d.\ with $y_i=f^*(x_i)+\xi_i$, where
$\EE[\xi_i\mid x_i]=0$ and $\xi_i$ is conditionally $\sigma$-sub-Gaussian. 
Let $\mathcal F$ be a symmetric, star-shaped class equipped with a Hilbert norm $\|\cdot\|_{\mathcal F}$. Consider the penalized least squares estimator
\[
\hat f\in\arg\min_{f\in\mathcal F}
\left\{\frac{1}{2n}\sum_{i=1}^n (y_i-f(x_i))^2+\lambda_n\|f\|_{\mathcal F}^2\right\}.
\]
Suppose $f^*\in\mathcal F$ and $\|f^*\|_{\mathcal F}\le R$. Define the localized difference class
\[
\mathcal G:=\{\Delta=f-f^*: f\in\mathcal F,\ \|f\|_{\mathcal F}\le R\}.
\]
and the local empirical Rademacher complexity
\[
\widehat{\mathcal R}_n(\mathcal G,\delta)
:=
\EE_{\varepsilon}\Bigg[
\sup_{\Delta\in\mathcal G:\|\Delta\|_{2,n}\le \delta}
\Big|\frac{1}{n}\sum_{i=1}^n \varepsilon_i \Delta(x_i)\Big|
\ \Big|\ x_{1:n}\Bigg].
\]
Let $\bar\delta_n$ be the upper bound of the critical radii satisfying $\widehat{\mathcal R}_n(\mathcal G,\bar\delta_n) \le \frac{\bar\delta_n^2}{32\sigma}$,
and define $\delta_n = \bar\delta_n + c_0\sigma\sqrt{\frac{\log(c_1/\zeta)}{n}}$ for some numerical constants $c_0,c_1>0$.
Assume that $\lambda_n \ge\frac34 \delta_n^2$, then there exist constant $C_1>0$ such that, with probability at least
$1-\zeta$,
\[
\|\hat f-f^*\|_{2}^2
\le C_1 R^2\big(\delta_n^2+\lambda_n\big).
\]
\end{lemma}

\end{document}